\newif\ifmam
\mamfalse

\documentclass{amsart}
\usepackage{amssymb,amsfonts}
\usepackage{enumerate}
\usepackage{mathrsfs}
\usepackage{tikz}
\usepackage{tikz-cd}
\usepackage{algorithm}
\usepackage{algorithmic}
\usepackage{placeins}
\usepackage{hyperref}
\usepackage{biblatex}
\usepackage{caption}
\usepackage{subcaption}
\usepackage[textsize=small]{todonotes}

\numberwithin{equation}{section}

\newtheorem*{thm*}{Theorem}

\newtheorem{prop}[equation]{Proposition}

\theoremstyle{definition}
\newtheorem{defn}[equation]{Definition}

\newtheorem*{alg}{Algorithm}

\theoremstyle{remark}

\newcommand{\bR}{\mathbb{R}}

\newcommand{\aD}{\mathcal{D}}
\newcommand{\aE}{\mathcal{E}}
\newcommand{\Aff}{\mathcal{A}f\!f(d)}
\newcommand{\Id}{\mathrm{Id}}

\title{Resampling and averaging coordinates on data}

\author[A. J. Blumberg]{Andrew J. Blumberg}
\address{Irving Institute for Cancer Dynamics \\ Departments of Mathematics and Computer Science \\ Columbia University, NY}
\email{andrew.blumberg@columbia.edu}
\thanks{The first author was partially supported by the NSF grant
  DMS-1912194 and by ONR grant N00014-22-1-2679.}

\author[M. Carri\`ere]{Mathieu Carri\`ere}
\address{DataShape, Centre Inria d'Universit\'e d'Azur \\ Biot, France}
\email{mathieu.carriere@inria.fr}

\author[J. H. Fung]{Jun Hou Fung}
\address{Department of Systems Biology \\ Columbia University, NY}
\email{jf3380@cumc.columbia.edu}
\thanks{The third author was supported by the NSF under grant DMS-1912194.}

\author[M. A. Mandell]{Michael A. Mandell}
\address{Department of Mathematics \\ Indiana University, IN}
\email{mmandell@iu.edu}
\thanks{The fourth author was supported by the ONR grant N00014-22-1-2675}

\begin{document}

\begin{abstract}
We introduce algorithms for robustly computing intrinsic coordinates
on point clouds.  Our approach relies on generating many candidate
coordinates by subsampling the data and varying hyperparameters of the
embedding algorithm (e.g., manifold learning).  We then identify a
subset of representative embeddings by clustering the collection of
candidate coordinates and using shape descriptors from topological
data analysis.  The final output is the embedding obtained as an
average of the representative embeddings using generalized Procrustes
analysis.  We validate our algorithm on both synthetic data and
experimental measurements from genomics, demonstrating robustness to
noise and outliers.
\end{abstract}

\maketitle


\section{Introduction}

A central postulate of modern data analysis is that the
high-dimensional data we measure in fact arises as samples from a
low-dimensional geometric object.  This hypothesis is the basic
justification for dimensionality reduction which is the first step in
almost all geometric data analysis algorithms, particularly in
clustering and its higher dimensional analogs.  This is aimed at
mitigating the ``curse of dimensionality'', a broad term for various
concentration of measure results that imply that the geometry of
high-dimensional Euclidean spaces behaves very counter-intuitively.
The curse of dimensionality tells us that if this hypothesis is not
satisfied, we cannot expect to perform any meaningful inference at
all.

The recent explosion of data acquisition processes in many different
scientific fields (e.g., single-cell RNA sequencing experiments in
computational biology, realistic synthetic data sets obtained from
deep generative models, multivariate time series generated from large
numbers of sensors, etc.) has led to a dramatic increase of publicly
available data sets in high ambient dimensions. The need for tractable
and accurate data science tools for processing these data sets has
thus become critical. While supervised machine learning and deep
learning methods have proven to be very efficient in many different
areas and applications of data science, they are often limited by the
difficulty of finding suitable training data.  Hence, the problem of
studying this data via dimensionality reduction and exploratory data
analysis has become central.

However, dimensionality reduction algorithms generally have the
inherent complication that they depend on a variety of ad hoc
hyperparameters and are sensitive both to noise concentrated around
the underlying geometric object and to outliers.  Moreover, all methods are known to be very
sensitive to these hyperparameters: a slight change in only one
parameter can lead to dramatic changes in the output lower-dimensional
embeddings.  As a consequence, it is challenging to perform meaningful
inference based on such procedures and most serious applications
involve a lot of ad hoc cross-validation procedures.

This article proposes a method to resolve this issue by producing {\em
robust} embeddings, employing the following process: 
\begin{itemize}
\item Subsampling and varying hyperparameters to produce a number of
embeddings of a given (low) dimension using dimensionality
reduction techniques,
\item using affine isometries and the solution to the general
``Procrustes problem'' to align the embeddings in a way that minimizes
a certain natural metric (explained below),
\item clustering the aligned embeddings based on that metric,
\item identifying a cluster of ``representative'' embeddings by
  looking at the cluster density and using topological data analysis
  (TDA) to eliminate clusters with topologically complex embeddings, and
\item taking the average (centroid) of the embeddings in the
  representative cluster to produce a final low dimensional embedding
  of most of the points.  (Points which are not present in the
  averaged embeddings are returned as ``outliers''.)
\end{itemize}
The subsampling leads to results that are robust with respect to
isotropic noise and have reduced sensitivity to outlier data points.  The
clustering discards outlier embeddings with a high level of
distortion.  Essentially any dimensionality reduction algorithm can be
used, but the procedure works best when the algorithm preserves some
global structure; e.g., t-SNE and UMAP, which preserve local
relationships but can radically distort global relationships produce
less sharp results.  (See Section~\ref{sec:real} for further discussion of
this point.)  The intuition behind our use of TDA invariants (notably
persistent homology) to detect the representative embeddings is that
we expect coordinate charts to be contractible subsets: 
theoretical guarantees for manifold learning imply that the algorithms
really only work in this case.  Persistent
homology is a convenient way to detect complicated topological
features such as holes.

We illustrate the process in Figure~\ref{fig:intro}. The
figure shows a standard 3-dimensional ``Swiss roll'' synthetic data
set and uses Isomap to produce a
$2$-dimensional coordinate chart.  This is a particularly easy
example, but we achieve similar results in the presence even of a
large number of outliers and across parameter regimes.  See
Section~\ref{sec:synthetic} for a variety of synthetic examples.

    \begin{figure}
        \centering
        \begin{subfigure}[t]{0.45\textwidth}
            \centering
            \includegraphics[height=1.6in]{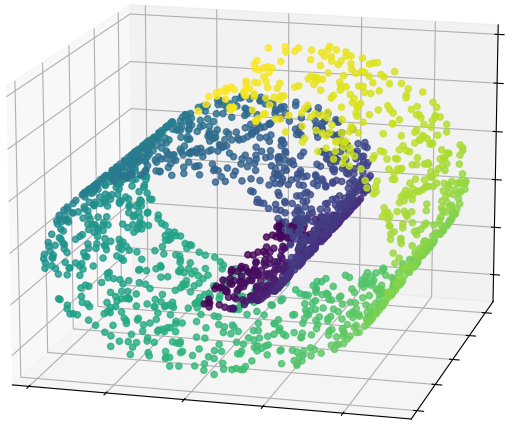}
            \caption{Original data in 3D}
        \end{subfigure}%
        \begin{subfigure}[t]{0.45\textwidth}
            \centering
            \includegraphics[height=1.6in]{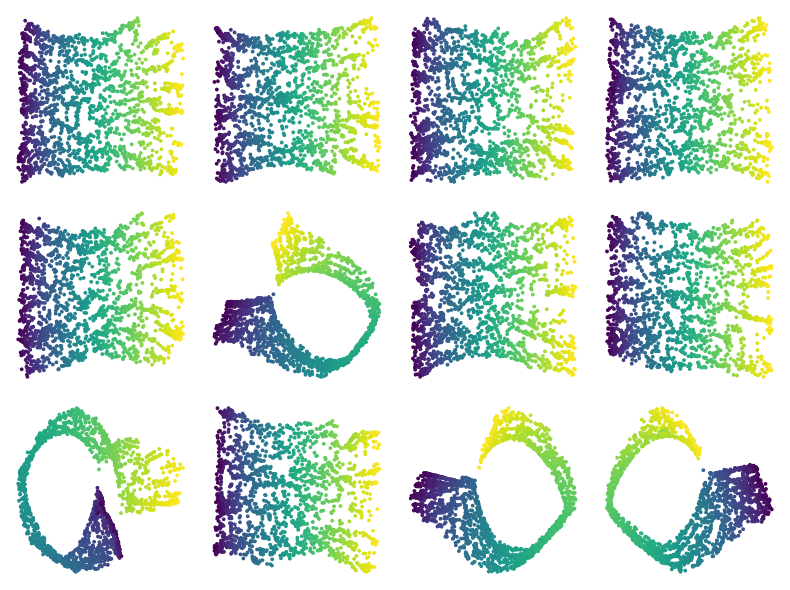}
            \caption{Isomap results on subsamples}
        \end{subfigure}
        \\
        \begin{subfigure}[t]{0.45\textwidth}
            \centering
            \includegraphics[height=1.6in]{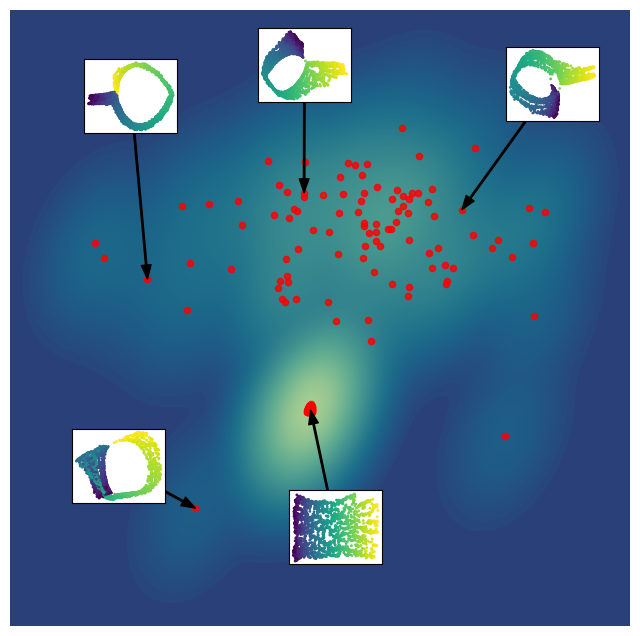}
            \caption{Clustering embeddings}
        \end{subfigure}
        \begin{subfigure}[t]{0.45\textwidth}
            \centering
            \includegraphics[height=1.6in]{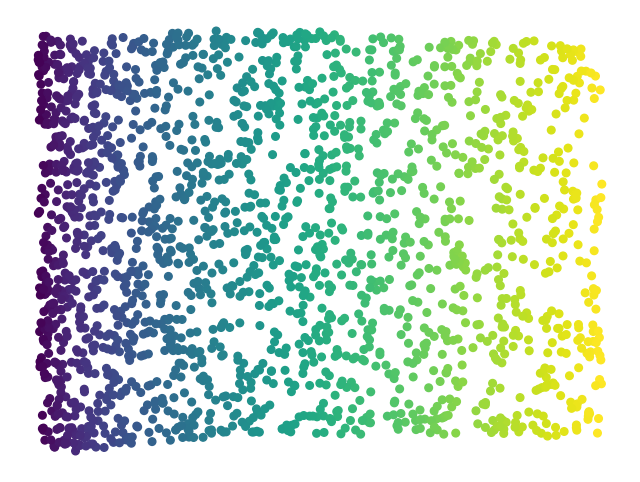}
            \caption{Final aligned output}
        \end{subfigure}
        \caption{Swiss roll dataset}
        \label{fig:intro}
    \end{figure}

The discussion above emphasizes the robustness in the presence of
outliers, but our technique has an additional use case where it can be
employed to reduce the computational complexity of dimensionality reduction
computations: instead of computing a dimensionality reduction of the entire
data set, the procedure above can be used to produce a global set of
coordinates by averaging coordinates from much smaller subsamples.

In order to demonstrate the use of our procedure on real data, we
apply it to multiple dimensionality reduction methods for producing coordinates on genomic data sets in Section~\ref{sec:real},
specifically blood cells (PBMC) and mouse neural tissue.  The results
show that our procedure can be used to produce robust coordinates from
some manifold learning algorithms and indicates the instability of
coordinates produced from others.

\subsection*{Related work}

Procrustes analysis has long been used to study
shapes~\cite{Kendall1984}.  Building on scaling methods, the
orthogonal Procrustes problem was first studied and solved by
Green~\cite{Green1952}, and Sch\"{o}nemann and
Sch\"{o}nemann-Carroll~\cite{Schonemann1966,Schonemann1970}, and later
this was generalized to three or more shapes by
Gower~\cite{Gower1975}.  Nowadays, there are many variants of
Procrustes problems.  Beyond statistical shape
analysis~\cite{Goodall1991}, Procrustes analysis has found
applications in sensory science~\cite{Dijksterhuis1995a}, market
research~\cite{Gower1994}, image analysis~\cite{Glasbey1998},
morphometrics~\cite{Bookstein1992}, protein structure
\cite{McLachlan1972}, environmental science~\cite{Richman1993},
population genetics~\cite{Wang2010},
phylogenetics~\cite{Balbuena2013}, chemistry~\cite{Andrade2004},
ecology~\cite{Lisboa2014}, and more.  In forthcoming work, the authors
will present an application to neuroscience. 

The early applications of Procrustes analysis were in the field of
psychometrics, where point sets are registered to a single ``reference''
profile that is sometimes realized in physical space, and datasets are
small.  More recently, it has been increasingly recognized that the same
techniques and ideas can be applied just as well to the shape of
``data'' itself with no expectation of a reference whatsoever, and to
high-dimensional datasets~\cite{Andreella2022} with many samples.
For example,~\cite{Ma2024} uses a version of Procrustes analysis with
shuffling to align low-rank representations of single cell
transcriptomic data.  Procrustes analysis can also be applied to
statistics itself.  In this sense, part of our current study can be
viewed as an expansion of~\cite{deLeeuw1986}, which studied a
jackknife procedure on multidimensional scaling, a limited form of
manifold learning.  Similarly,~\cite{Wang2008} applied Procrustes
analysis to the outputs of dimensionality reduction, specifically
Laplacian eigenmaps and locality preserving projections, but only for
two point clouds at a time. 

\subsection*{Summary} In Section~\ref{sec:background}, we provide a
concise review of the background for our work; we discuss manifold
learning and dimensionality reduction algorithms, the Procrustes
distance and alignment problem, and topological data analysis.  Next, in
Section~\ref{sec:algo}, we describe our algorithm.  We then begin a
theoretical analysis of the behavior of our algorithm by reviewing in
detail the solution to the generalized Procrustes problem in
Section~\ref{sec:procrustes}.  We employ a solver based on alternating
least squares minimization, and we describe the stability and
convergence behavior of this process in Section~\ref{sec:ALSanal}.  We
then use this to analyze the stability of our algorithm in
Section~\ref{sec:stability}.  The paper concludes with two sections
exploring the behavior of our algorithm: Section~\ref{sec:synthetic}
studies how it works on simulated data with various kinds of noise and
outliers, and Section~\ref{sec:real} applies our algorithm to
single-cell RNAseq data.

\subsection*{Code availability}

For aligning two point clouds, we used the command
\texttt{scipy.spatial.procrustes} from the SciPy library
\cite{SciPy2020}.  For general Procrustes alignment of two or
more point clouds, we developed our own software implementation based
on alternating minimization, which can be found at
\url{github.com/jhfung/Procrustes}.

\subsection*{Acknowledgments}

We thank Abby Hickok and Raul Rabadan for useful comments.

\section{Background about manifold learning, topological data
  analysis, and Procrustes distance}
\label{sec:background}

Our algorithm uses as building blocks three fundamental concepts from
geometric data analysis.  The first building block is dimensionality
reduction. Our algorithm takes as a black box a choice of dimensionality
reduction algorithm.  We use manifold learning for this purpose, and
to be concrete, we will focus on Isomap, but many other manifold
learning (or more general dimensionality reduction algorithms) could
be used instead.  Manifold learning seeks to find intrinsic
coordinates on data that reflects the shape of the data.  The second
building block is persistent homology, which is the main invariant of topological
data analysis (TDA).  Persistent homology is a multiscale shape
descriptor that provides qualitative shape information; we will use it
to detect when coordinate charts are spread out and have contractible components.  The
third building block is the Procrustes distance, which is a matching distance for
(partially defined) vector-valued functions on a given finite set.  Computing the
Procrustes distance involves computing an optimal alignment of the
vectors; our algorithm uses both the distance as a metric and also the
computation of the distance as an alignment algorithm.

The manifold learning algorithm takes a point cloud $X\subset \bR^{N}$
and produces an embedding $X\to \bR^{d}$ for $d \ll N$ meant to retain
the intrinsic relationship between the points of $X$.  Specifically,
for Isomap, the procedure is as follows.  Given a point
cloud $X \subset \bR^N$ and a scale parameter $\epsilon > 0$ (or a
nearest neighbor limit $\ell$), we form
the weighted graph $G$ with vertex set labelled by the points $X$ and
an edge of weight $\lVert x - x'\rVert$ when $\lVert x - x'\rVert < \epsilon$ (or when
$x'$ is one of the $\ell$ closest points to $x$).  The
graph metric on $G$ now provides a new metric on the point cloud $X$.
Finally, we use MDS (multidimensional scaling) to embed this new
metric space into $\bR^d$, for $d \ll N$. The intuition is that when
the data $X$ comes from a low dimensional manifold embedded in a high
dimensional space, short paths in the ambient space accurately reflect
an intrinsic notion of nearness in the manifold.  When the points $X$
have been uniformly sampled from a convex subset of a manifold such
that $\epsilon$ is sufficiently small relative to the curvature of the
manifold and the injectivity radius of the embedding (which can be
described in terms of the reach or condition number of the manifold),
the coordinates produced by this procedure do approximate the intrinsic
Riemannian metric on $X$.

Topological data
analysis uses invariants from algebraic topology that encode
qualitative shape information to give a picture of the intrinsic
geometry of a point cloud.  The $q$th homology of a topological space $T$ is a
vector space which encodes the $q$-dimensional holes in $T$; for
example, when $q=0$ this is measuring the number of path-connected
components of the space and when $q=1$ it is counting the number of
circles that are not filled in.  Topological data analysis assigns to
a point cloud a family of associated simplicial complexes, usually
filtered by a varying feature scale parameter.  The resulting
algebraic invariant, persistent homology, captures the scales at which
homological features appear and disappear.  In the case when $q=0$,
the persistent homology is basically the single-linkage hierarchical
clustering dendrogram.  When $q=1$, persistent homology measures the
number of loops in the underlying graph of the data that cannot be
filled in at each scale.

Since manifold learning only really makes sense for relatively evenly
sampled points from subsets whose components are contractible (or in the case of
Isomap, the stronger condition of being convex), we can use
topological data analysis to measure how well the resulting chart
appears to satisfy this condition.  Specifically, a good chart will
have a large cluster for each component in the $PH_0$ dendrogram and will have
only very small (noise) loops in $PH_1$.

The Procrustes distance we use comes from the Procrustes
problem.  In basic form, given two $d\times n$ matrices $X$ and $Y$,
thought of as two functions from $\{1,\dotsc,n\}$ to $\bR^{d}$, the
Procrustes distance is the minimum of the $\ell^{2}$ distances between the
functions after applying an affine isometry $g$ to one of them:
\[
\aD(X,Y)=\min_{g} d(g\cdot X,Y)=\min_{g} \left({\sum_{j=1}^{n}d(g\cdot X(j),Y(j))^{2}}\right)\raisebox{1.25ex}{\mbox{$^{1/2}$}}.
\]
More generally, we consider the case when $X$ and $Y$ are partially
defined functions on $\{1,\dotsc,n\}$.  If we let $I_{X}$ and $I_{Y}$
denote the subsets of $\{1,\dotsc,n\}$ where $X$ and $Y$
(respectively) are defined, then the Procrustes distance is
\[
\aD(X,Y)=\min_{g} \left({\sum_{j\in
I_{X}\cap I_{Y}}d(g\cdot X(j),Y(j))^{2}}\right)\raisebox{1.25ex}{\mbox{$^{1/2}$}}.
\]
The algorithm for calculating $\aD$ finds an affine isometry
$g$ minimizing the distance; we review it in
Section~\ref{sec:procrustes}.  We note that in the partially defined
case, the resulting Procrustes distance $\aD$ is not a metric; however, for
functions with substantial overlap in domain, it does provide a 
measure of dissimilarity on their overlap.

Our algorithm uses a generalization of the Procrustes
distance algorithm to search for optimal rotations to align 
different embeddings.  The algorithm for two matrices always
finds an optimal rotation (and therefore accurately calculates the
Procrustes distance) and in principle gives an exact solution.  The
algorithm to align more than two matrices is an iterative optimization
procedure, which always converges, and appears to generically converge
to the optimal solution, but is not guaranteed to converge to it. 

\section{The algorithm}\label{sec:algo}

This section defines the basic algorithm we propose in this paper.  It
admits several choices of blackbox subroutines and tolerance parameters.
The first major blackbox subroutine is a procedure \textbf{DimRed}
for dimensionality reduction: for $X \subset \bR^N$ produces an
embedding $X \to \bR^d$ (where $d \ll N$).  We let $\Theta$ denote the
collection of parameters controlling the behavior of the
dimensionality reduction procedure.  For sake of discussion, we take
the procedure to be Isomap, which has parameter the neighborhood size
$\epsilon$.

A second major blackbox subroutine is a sampling procedure \textbf{Samp} that generates
random subsamples $Y \subseteq X$ of a given size $n$.  Typically, we
draw $Y_i$ uniformly and independently without replacement from $X$.

A third major blackbox subroutine is a procedure $\textbf{Param}$ that chooses the
parameter $\Theta$ for \textbf{DimRed}.  In the case of Isomap,
typically we take the parameter over a mesh of reasonable values.

The final major blackbox subroutine is a clustering algorithm
\textbf{Clust} for
finite sets with dissimilarity measures.

We describe the various minor parameters controlling the number of
iterations, tolerances for certain optimization procedures, number of
subsamples, etc., as they occur below.

Given the parameter choices, the input to the algorithm is a point cloud
\[
X=\{X(j)\mid j=1,\dotsc,n\}\subset \bR^{N}.
\]

\subsection*{Step 1}
We use 
$\textbf{Samp}$ to generate many subsamples $\{Y_a\}$ and apply
\textbf{DimRed} with parameter settings from \textbf{Param} to produce
embeddings $\phi_{a,b}\colon Y_a \to \bR^d$.  Let $X_{a,b}$ be the
subset of $\bR^{d}$ specified by the image, where the elements of each
subsample configuration $X_{a,b}$ inherit the indexing of
original data set $X$.  That is, each $X_{a,b}$ is indexed by the subset
of $\{1,\dotsc,n\}$ corresponding to the points of $X$ represented
in the subsample $Y_{a}$. 

\subsection*{Step 2}
We calculate the pairwise distances between the images $X_{a,b}$
using the Procrustes distance to define a dissimilarity measure on the
set of subsample configurations $\{X_{a,b}\}$ a pseudo-metric space.

\subsection*{Step 3}
We use \textbf{Clust} to form clusters.  

\subsection*{Step 4}
Choose a distinguished cluster (the ``good cluster'') from
among these as follows.  
\begin{itemize}
\item 
We discard clusters whose median inter-point
distance is above a certain tolerance (a minor parameter of the
algorithm).  We refer to the remaining clusters as ``dense clusters''.\\

\item
For each
dense cluster we select random elements in the cluster (the number
or percentage size of the selection a minor parameter of the
algorithm) on which to calculate the persistent homology 
$PH_{1}$ and essential dimensionality (number of singular values above
given tolerance).\\

\item
We discard dense clusters where a selected element has
essential dimensionality less than $d$, and then we choose the good
cluster to be the one that minimizes the maximum length of bars in
$PH_{1}$.\\
\item
If there are no remaining clusters or the maximum length of
bars in $PH_{1}$ is above a certain tolerance (a minor parameter of
the algorithm), we terminate with an error.
\end{itemize}

\subsection*{Step 5}
We discard all the subsets $X_{a,b}$ not in the good
cluster, and singularly reindex the subsets in the good cluster
$X_{i}=X_{a_{i},b_{i}}$ for $i=1,\dotsc,k$ (with $k$ the number of
$X_{a,b}$ in the good cluster).

\subsection*{Step 6}
We align the embeddings $X_{i}$ by applying an affine
isometry $X_{i}(j)\mapsto Q_{i}X_{i}(j)+v_{i}$ to each point
$X_{i}(j)$ in $X_{i}$ (for the indices $j\in \{1,\dotsc,n\}$ that
occur in $X_{i}$) where $Q_{i}$ is a $d\times d$ orthogonal matrix and
$v_{i}$ is a vector in $\bR^{d}$, chosen by the {\em alternating least
squares method} described in Section~\ref{sec:procrustes}.

\subsection*{Output}
The final output is an appropriate average of the aligned embeddings as
follows. For each point $X(j)$ in $X$ ($j=1,\dotsc,n$):
\begin{itemize}
\item If no $X_{i}$ includes a $j$th point, then the $j$th point is omitted
from the final embedding; otherwise,
\item The final embedding has $j$th point the average in $\bR^{d}$ of
the $j$th points of the $X_{i}$ which have $j$th points,
\[
\bar X(j)=\tfrac1s(X_{i_{1}}(j)+\dotsb+X_{i_{s}}(j)).
\]
\item We also output a list of ``outliers'' that consists of the omitted
indices, the indices that did not have points in any of the embeddings
$X_{i}$ chosen to produce the final average.
\end{itemize}

\section{A review of the Procrustes problem}\label{sec:procrustes}

Our main ingredient for averaging coordinates is posing the averaging
in terms of the so-called {\em Procrustes problem}, an alignment
problem that has been studied extensively in the psychometrics 
literature.  The material in this section distills the discussion
in~\cite{TenBerge1977, Commandeur-Matching,
TenBergeKiersCommandeur}. 

\subsection*{The standard orthogonal and affine orthogonal Procrustes
problem}

In its most standard form, given two matrices $X,Y$ of the same shape,
this problem is a matrix approximation problem that aims at finding
the best orthogonal matrix that matches $X$ to $Y$: 
\begin{equation}\label{standard-procrustes}
    Q_{\min} = \underset{Q\ {\rm s.t.}\ Q^T  Q = 1}{{\rm argmin}} \|Q X - Y\|_F,
\end{equation}
where $\|\cdot\|_F$ denotes the Frobenius norm (the square root of the
sum of the squares of the entries). By interpreting
$Q_{\min}$ as an isometry (in Euclidean space), and $X,Y$ as point
clouds (each column representing a point), the standard Procrustes
problem can thus be seen as an alignment problem which aims at finding
the best orthogonal transformation that transports the point cloud $X$
closest to the point cloud $Y$ in terms of the sum of pointwise
distances. In practice, it can reformulated as finding the best
orthogonal matrix $Q_{\min}$ approximating the given matrix $YX^{T}$,
and proved to be efficiently solved by computing the SVD of 
\[
YX^T=U\Sigma V^{T}
\]
(with $U$ and $V$ orthogonal and $\Sigma$ non-negative diagonal)
and taking $Q_{\min}=UV^T$.

If we allow the more general affine isometries (that allow a
translation component), it is easy to check that 
\begin{equation}\label{eq:affineprocrustes}
    \Omega_{\min} = \underset{\Omega\in \Aff}{{\rm argmin}} \|\Omega X - Y\|_F,
\end{equation}
is given by the affine isometry
\[
\Omega_{\min}x=Q(x-a)+b=Qx+(b-Qa)
\]
where $a$ is the centroid (mean) of $X$, $b$ is the
centroid of $Y$, and $Q$ is the orthogonal matrix $Q_{\min}$ that
solves equation~\eqref{standard-procrustes} for the matrices
$X-\mathbf{1}a$ and $Y-\mathbf{1}b$ (where $\mathbf{1}$ is the
$1\times n$ matrix of $1$s: the matrices $X-\mathbf{1}a$ and
$Y-\mathbf{1}b$ are the translations of $X$ and $Y$ to be centered on
the origin.)

In the notation of Section~\ref{sec:background}, the Procrustes
distance $\aD(X,Y)$ from $X$ to $Y$ is then 
\[
\aD(X,Y)=\lVert\Omega_{\min}X-Y\rVert_{F}.
\]

As discussed in Section~\ref{sec:background}, we need to consider the
more general case when the matrices are missing columns; we refer to
this as the \emph{missing points} case.  Viewing $d\times n$ matrices
as functions $\{1,\dotsc,n\}$ to $\bR^{d}$, the missing points case is
when we have partially defined
functions $X$ and $Y$ from $\{1,\dotsc,n\}$ to $\bR^{d}$; that is $X$
and $Y$ are functions from subsets $I_{X}$ and $I_{Y}$ of 
$\{1,\dotsc,n\}$ to $\bR^{d}$.  Let $I$ be the natural reindexing of
the intersection of the domains $I_{X}\cap I_{Y}$.  Then the
Procrustes distance $\aD(X,Y)$ as defined in
Section~\ref{sec:background} is calculated by the formula
\[
\aD(X,Y)=\lVert\Omega_{\min}X_{I}-Y_{I}\rVert_{F}
\]
where $\Omega_{\min}$ solves~\eqref{eq:affineprocrustes} for the
matrices $X_{I}$, $Y_{I}$ (the matrices obtained by restricting $X$
and $Y$ to $I_{X}\cap I_{Y}$).

\subsection*{Background on the generalized Procrustes problem}

The formulation of the Procrustes problem in the previous subsection
involved only 2 matrices or matrices missing points (partially defined
functions).  In this subsection, we begin the discussion of the problem
for multiple matrices.  We then generalize to the missing points case
and discuss algorithms in the following subsections.

We begin with some notation.  Let $X_1, X_2, \ldots, X_k$ be a
collection of $k$ point clouds, each containing $n$ \emph{ordered}
points in $\mathbb{R}^d$.  We can represent each $X_i$ as a $d \times
n$ matrix, and denote the collection of such by $\mathbf{X} = (X_i)$.
The allowable transformations will be drawn from a group $G$, which we
generally take to be the group of linear isometries (orthogonal transformations) or affine
isometries of $\bR^{d}$. We write 
$g_i \in G$ for the transformation that will be applied to the $i$th
point cloud $X_i$, and $\mathbf{g} = (g_i)$ for the collection of
them.

\begin{defn}
The \emph{generalized Procrustes problem} is the following: given $\mathbf{X}$, a
set of $k$ input configurations of $n$ points in $\bR^{d}$ and group
$G$ acting continuously on $\bR^{d}$, determine the optimal transformations $\mathbf{g}$ and configuration $Z$ that minimizes the loss functional 
    \begin{equation*}
        \mathcal{E}(\mathbf{X}, \mathbf{g}, Z) = \frac{1}{k}
	\sum_{i=1}^k \lVert g_i \cdot X_i- Z\rVert^2_{F}
=\frac1k\sum_{i=1}^{k}\sum_{j=1}^{n}\lVert g_{i}\cdot X_{i}(j)-Z(j)\rVert^{2}
    \end{equation*}
\end{defn}

The existence of solutions to Procrustes problems in the groups we are
interested in follows from simple considerations.  

\begin{prop}
    If $G$ is compact or the semi-direct product of a compact group
    and the translation group, then $\mathcal{E}$ achieves a global minimum.
\end{prop}

The solution is never unique when $G$ contains non-identity isometries
of $\bR^{d}$ because for any $\mathbf{X}, \mathbf{g}, Z$, and
any isometry $h\in G$, writing $h\mathbf{g}$ for $(hg_{i})$, we
have
\[
\mathcal{E}(\mathbf{X}, \mathbf{g}, Z)=
\mathcal{E}(\mathbf{X}, h\mathbf{g}, h\cdot Z).
\]
In the case when $G$ a subgroup of the isometries of $\bR^{d}$, we can
eliminate this trivial duplication of solutions by considering only
solutions that have $g_{1}$ a fixed element of $G$, possibly depending
on $\mathbf{X}$.  An obvious
choice is to take $g_{1}$ to be the identity, but the choice we take
below is to choose $g_{1}$ to be the translation that centers $X_{1}$
on $0\in \bR^{d}$.  The \emph{first fixed formulation} of the
generalized Procrustes problem (for $G$ the group of affine
isometries) is to find $\mathbf{g}, Z$ minimizing the loss
functional and satisfying the further condition that $g_{1}$ is the
translation that centers $X_{1}$ on $0\in \bR^{d}$.

We now specialize to the case where $G$ is the group of affine
isometries $\Aff$ (consisting of composites of translations, rotations, and reflections in
$\mathbb{R}^d$).  We specify an element $g$ of $G$ by a $d\times d$
orthogonal matrix $Q$ and a vector $v$, where for $x\in \bR^{d}$,
\[
g\cdot x=Qx+v.
\]
On a configuration $X$, viewed as a $d\times n$ matrix, the action is 
given by matrix multiplication and addition
\[
g\cdot X=QX+\mathbf{1}v
\]
where $\mathbf{1}$ denotes the $1\times n$ matrix of $1$s.

A key observation is that we can eliminate translations and $Z$
as variables in the generalized Procrustes problem: 

\begin{prop}\label{prop:elimvar}
For fixed $\mathbf{X}$, the minimum value of
$\aE(\mathbf{X},\mathbf{g},Z)$ occurs at elements
$\mathbf{g}$, $Z$ where the
centroid of each $g_{i}X_{i}$ is the origin and each $Z(j)$ is the average of
$X_{i}(j)$,
\[
Z(j)=\frac1k\sum_{i=1}^{k}g_{i}\cdot X_{i}(j). 
\]
\end{prop}

\begin{proof}
We begin with the translation part.
For fixed
$\mathbf{X},\mathbf{g}$, $Z$, let $q\in 
1,\dotsc,k$, let $v\in \bR^{d}$, and consider the path $\mathbf g(t)$
in $G^{k}$ where $g_{i}(t)=g_{i}$ for $i\neq q$, and $g_{q}(t)$ is the
composite of $g_{q}$ followed by the translation $x\mapsto x+tv$.
Then
\begin{align*}
\left.\frac{d}{dt}\right|_{t=0}\hspace{-1em}
\mathcal{E}(\mathbf{X},\mathbf{g}(t),Z)
&=\frac2k\sum_{i=1}^{k}\sum_{j=1}^{n}\left( (g_{i}\cdot X_{i}(j)-Z(j))\cdot 
\left.\frac{d}{dt}\right|_{t=0}\hspace{-1em}
(g_{i}(t)\cdot X_{i}(j)-Z(j))\right)\\
&=\frac2k\sum_{j=1}^{n}(g_{q}\cdot X_{q}(j)-Z(j))\cdot v.
\end{align*}
If $\mathbf{g},Z$ gives a minimum for
$\mathcal{E}(\mathbf{X},\mathbf{g},Z)$ then this derivative must
be zero for every $q$ and $v$, which implies that 
\[
\sum_{j=1}^{n}Z(j)=\sum_{j=1}^{n} g_{q}\cdot X_{q}(j)
\]
for all $q$, and so all the $g_{i}\cdot X_{i}$ and $Z$ must have the same
centroid.  Thus, the minimum can only occur for a $g_{i}$ that centers
$X_{i}$ on the centroid of $Z$. In the first fixed formulation,
the common centroid of $Z$ and the $X_{i}$ is then the origin.
More generally, taking advantage of the symmetry of the problem as a
whole, given a solution $\mathbf{g},Z$, there exists a solution
with the common centroid at the origin by applying an appropriate
translation.

To eliminate $Z$ as a variable, if we keep $\mathbf{g}$ fixed
and for some $r\in 1,\dotsc,n$, let $Z_{t}$ be the path with
$Z_{t}(j)=Z(j)$ for $j\neq r$ and $Z_{t}(r)=Z(j)+tv$, we
then have
\begin{align*}
\left.\frac{d}{dt}\right|_{t=0}\hspace{-1em}
\mathcal{E}(\mathbf{X},\mathbf{g},Z_{t})
&=\frac2k\sum_{i=1}^{k}\sum_{j=1}^{n}\left( (g_{i}\cdot X_{i}(j)-Z(j))\cdot 
\left.\frac{d}{dt}\right|_{t=0}\hspace{-1em}
(g_{i}\cdot X_{i}(j)-Z_{t}(j))\right)\\
&=\frac2k\sum_{i=1}^{k}(g_{i}\cdot X_{i}(r)-Z(r))\cdot v.
\end{align*}
In the case when $\mathbf{g},Z$ gives a minimum for the loss
functional, we then conclude
\[
Z(r)=\frac1k\sum_{i=1}^{k}g_{i}\cdot X_{i}(r)
\]
for all $r$.  The minimum can only occur when $Z(j)$ is the
average over $i$ of the $g_{i}\cdot X_{i}(j)$.
\end{proof}

By Proposition~\ref{prop:elimvar}, we do not need to search over the
space of $Z$ and we can restrict to searching for the orthogonal
transformation part of the $g_{i}$. This leads to the centered
formulation of the problem.

\begin{defn}\label{defn:centered}
The \emph{centered formulation} of the generalized Procrustes problem
for the group of affine isometries is the following: given $\mathbf{X}$, a
set of $k$ input configurations of $n$ points in $\bR^{d}$ whose
centroids are the origin, determine the optimal orthogonal
transformations $\mathbf{Q}=(Q_{i})\in O(d)^{n}$ that minimize the loss functional 
    \begin{equation*}
        \mathcal{E}(\mathbf{X}, \mathbf{Q}) = \frac{1}{k} \sum_{i=1}^k
	\lVert Q_{i}X_i- Z\rVert^2_{F}
=\frac1k\sum_{i=1}^{k}\sum_{j=1}^{n}\lVert Q_{i}X_{i}(j)-Z(j)\rVert^{2}
    \end{equation*}
where $Z(j)=\frac1k\sum Q_{i}X_{i}(j)$.
\end{defn}

There is always a solution with $Q_{1}=\Id$; the centered first fixed
formulation is to find the optimal $\mathbf{Q}$ subject to the
restriction that $Q_{1}=\Id$.

In the centered formulation, because $Z$ is the mean of
the $Q_{i}X_{i}$, we can rewrite the loss functional in the following
form, which is sometimes more convenient:
\begin{equation}\label{eqn:procrustes_loss_final}
\mathcal{E}(\mathbf X,\mathbf Q)
=\frac1k\sum_{i=1}^{k}\bigl(\lVert Q_{i}X_{i}\rVert^{2}_{F}-
\lVert Z\rVert^{2}_{F}\bigr).
\end{equation}
Expanding out the definition of $Z$, this is then:
\begin{equation}\label{eqn:procrustes_loss_final_var}
\mathcal{E}(\mathbf X,\mathbf Q)
=\frac1k\sum_{i=1}^{k}\biggl(\lVert Q_{i}X_{i}\rVert^{2}_{F}- \lVert\frac1k\sum_{i=1}^{k} Q_{i}X_{i}\rVert^{2}_{F}\biggr).
\end{equation}

\subsection*{The alternating least squares (ALS) method for the
generalized Procrustes problem}

The background in the previous subsection suggests the following
algorithm for searching for the solution to the generalized Procrustes
problem for the group of affine isometries.  In the centered
formulation, we iteratively use the 2 matrix Procrustes problem
solution applied to $X_{i}$ and $Z$, to get better and
better matches between each $Q_{i}X_{i}$ and the mean of the
$Q_{i}X_{i}$.  For the non-centered formulation, we first reduce to
the centered formulation by translating the original configurations to
be centered on the origin.

We state the algorithm in the centered formulation, where we are given the
input configurations $\mathbf{X}$, where we assume each $X_{i}$ is
centered on the origin, and we search for the orthogonal matrices
$\mathbf{Q}$ minimizing the loss functional $\aE$ of
Definition~\ref{defn:centered}. We assume a small number $\rm{tol}$ as
a pre-selected tolerance for termination and a large integer
$\rm{max\_iter}$ for a maximum number of iterations.

\begin{alg}[Basic ALS method]\ 
\begin{enumerate}
\item[Step 0.] Initialize $Z=\frac1k\sum_{i=1}^{k}X_{i}$.
\item[Step 1.] Set $\mathrm{loss}=\aE(\mathbf{X},\mathbf{Q})$. 
\item[Step 2.] Use SVD to solve $ZX_{i}^{T}=U_{i}\Sigma_{i}V_{i}^{T}$
(for $i=1,\dotsc,k$) where
$U_{i}$, $V_{i}$ are orthogonal $d\times d$ matrices and $\Sigma_{i}$
is a
non-negative diagonal matrix with diagonal entries in decreasing order.
\item[Step 3.] Update $Z=\frac1k\sum_{i=1}^{k}U_{i}V_{i}^{T}X_{i}$.
\item[Step 4.] If $|\mathrm{loss}-\aE(\mathbf{X},\mathbf{Q})|\geq \mathrm{tol}$
and the number of iterations is less than $\rm{max\_iter}$, iterate
from Step 1. Else:
\item[Step 5.] Return $(U_{i}V_{i}^{T}),Z$.
\end{enumerate}
\end{alg}

We have written the algorithm to emphasize the concept; it admits many
tweeks to increase efficiency, some of which are discussed after the
final version of the algorithm in the next subsection.

For stability of output, we should find some normalization of the
raw output.  For the first fixed formulation, we look for a solution
$(Q_{i}),Z$ with $Q_{1}=\Id$.  Given the raw output $(Q_{i}),Z$, we get a
transformed output $(Q_{1}^{-1}Q_{i}),Q_{1}^{-1}Z$ with the same
loss functional value but with the first orthogonal transformation the
identity. 

For empirical data where the different input configurations are not
far from being rotations and translations of each other, this
algorithm in practice converges to the solution to the generalized
Procrustes problem.  For independent Gaussian random input
configurations, there are many local minima for the loss functional
that are close to the absolute minimum, and this algorithm does not
always find the transformations giving the absolute minimum, but does
in practice find transformations with loss functional close to the
minimum.

The algorithm above admits some obvious criticisms.  First, while
generically it does solve the Procrustes problem for 2 input
configurations, there is a low dimensional space of inputs where the
algorithm fails.  For example, the algorithm fails when
$X_{2}=-X_{1}$.  The correct fix for this is not to apply this
algorithm with fewer than 3 input configurations, and use the precise
not iterative algorithm in that case.  More subtly, if the input
configurations are at a unstable critical point of the loss
functional, the algorithm will terminate immediately and not find a
minimum or local minimum.  The fix for this is to require a minimum
number of iterations before termination. 
 
While this is the obvious algorithm based on the discussion above, ten
Berge~\cite{TenBerge1977} does a deeper analysis of the generalized
Procrustes problem and finds an algorithm that in practice appears to
find transformations with smaller loss functional than the basic
algorithm a large percentage of the time on inputs where the loss
functional has a large number of local minima.  The basic algorithm
above implicitly uses the fact that when $\mathbf{Q},Z$ solves the
generalized Procrustes problem, then each $Z(Q_{i}X_{i})^{T}$ is
symmetric positive semi-definite (where we are writing $Q_{i}$ for
$U_{i}V_{i}^{T}$).  In terms of the algorithm, 
this is the result of the iteration because at Step 3, we have
\[
Z(Q_{i}X_{i})^{T}=ZX_{i}^{T}Q_{i}^{T}=(U_{i}\Sigma_{i}V_{i}^{T})(U_{i}V_{i}^{T})^{T}
=U_{i}\Sigma_{i}U_{i}^{T}.
\]
However, a solution to the generalized Procrustes problem actually
satisfies the more restrictive condition that $(Z-\frac1k
Q_{i}X_{i})(Q_{i}X_{i})^{T}$ is symmetric positive semi-definite.  This leads to the
following more sophisticated algorithm.

\begin{alg}[ALS method]\ 
\begin{enumerate}
\item[Step 0.] Initialize $Q_{i}=\Id$, $Z=\frac1k\sum_{i=1}^{k}X_{i}$.
\item[Step 1.] Set $\mathrm{loss}=\aE(\mathbf{X},\mathbf{Q})$. 
\item[Step 2.] For $i$ in $1,\dotsc,k$
\begin{enumerate}
\item Use SVD to solve $(Z-\frac1k
Q_{i}X_{i})X_{i}^{T}=U_{i}\Sigma_{i}V_{i}^{T}$ for $U_{i}$, $V_{i}$ orthogonal
$d\times d$ matrices and $\Sigma_{i}$ a non-negative diagonal matrix
with diagonal entries in decreasing order. 
\item Update $Q_{i}=U_{i}V_{i}^{T}$
\item Update $Z=\frac1k\sum_{i=1}^{k}Q_{i}X_{i}$.
\end{enumerate}
\item[Step 3.] If $|\mathrm{loss}-\aE(\mathbf{X},\mathbf{Q})|\geq \mathrm{tol}$
and the number of iterations is less than $\rm{max\_iter}$, iterate
from Step 1. Else:
\item[Step 4.] Return $(Q_{i}),Z$.
\end{enumerate}
\end{alg}

We note that in the case of 2 input configurations, this algorithm
always finds the solution to the Procrustes problem (but does an extra
step from the usual 2 input configuration algorithm).

Again, for stability, we should normalize the output, for example by
replacing $(Q_{i}),Z$ with $(Q_{1}^{-1}Q_{i}),Q_{1}^{-1}Z$.

\subsection*{The ALS method for the
generalized Procrustes problem with missing points}

In the context of missing points, the configurations $X_{i}$ only have
elements of $\bR^{d}$ specified for some but not necessarily all
indices $\{1,\dotsc,n\}$.  It is still convenient to represent $X_{i}$
as a $d\times n$ matrix, where we fill in the zero column for the for
the indices where $X_{i}$ is not defined.  To keep track of which
indices are defined in a manner conducive to easily expressed matrix
operations, we let $K_{i}$ denote the $n\times n$ diagonal matrix
which has diagonal entry 1 at the indices where $X_{i}$ is defined and
$0$ at the indices where $X_{i}$ is not defined.  In this case,
whenever $X_{i}'$ is any $d\times n$ that agrees with $X_{i}$ on the
columns for which $X_{i}$ is defined, $X_{i}=X_{i}'K_{i}$. (In other
words, if we always work with $X_{i}K_{i}$, it does not matter how we
fill in the columns where $X_{i}$ is not defined.)  We assume that the
configuration $X_{i}$ is non-empty, and therefore $K_{i}$ is not the
zero matrix.  We let $n_{i}$ denote the number of points in $X_{i}$;
then $0<n_{i}\leq n$, and $n_{i}$ is the sum of the entries of
$K_{i}$.

Let $K=\sum_{i=1}^{k}K_{i}$. Then $K$ is a diagonal matrix and the
diagonal entries indicate the number of configurations in which a
particular index occurs.  Without loss of generality, we can assume
that none of these diagonal entries is zero: if it is, we can drop
that index from consideration and reindex the problem as a whole.
Then $K$ is invertible and setting $k_{j}=K_{j,j}$, $j=1,\dotsc,n$, we
have $k_{j}>0$.

In this regime, the mean of the configurations should be calculated at
each index $j=1,\dotsc,n$ using only the configurations in which that
index appears:
\[
Z(j)=\frac1{k_{j}}\sum_{i=1}^{k}(X_{i}K_{i})(j)=(\sum_{i=1}^{k}X_{i}K_{i}K^{-1})(j)
\]
(where, generalizing the convention for the configurations $X_{i}$, we
are writing $Y(j)$ for the $j$th column of an arbitrary $d\times n$
matrix $Y$). 

When some configurations are missing points, we can no
longer center all the configurations at 0 and expect the mean to be
centered at 0, and we can no longer work in the centered formulation.
Specifying an affine isometry using a linear 
isometry $Q$ and a translation vector $v$, the loss function for a
particular choice of $\mathbf{Q}=(Q_{i}), \mathbf{v}=(v_{i})$ becomes
\[
\aE(\mathbf{X},\mathbf{K},\mathbf{Q},\mathbf{v})=
\frac1k\sum_{i=1}^{k}\lVert Q_{i}(X_{i}+\mathbf{1}v_{i})-Z\rVert^{2}_{F}
\]
where
\[
Z=\sum_{i=1}^{k}Q_{i}(X_{i}+\mathbf{1}v_{i})K_{i}K^{-1}.
\]
(Here as above $\mathbf{1}$ denotes the $1\times n$ matrix of $1$s.)

The ALS algorithm also needs to be modified to account for
translations, and since the output configurations may no longer be
centered at 0, we should no longer ask for the input configurations to be
centered at 0.

\begin{alg}[ALS method with missing points]\ 
\begin{enumerate}
\item [Step 0] Initialize constants:
\[
K=\sum_{i=1}^{k}K_{i}, n_{i}=\sum_{j=1}^{n}K_{j,j}, a_{i}=\frac1{n_{i}}\sum_{j=1}^{n}(X_{i}K_{i})(j)
\]
and variables:
\[
Q_{i}=\Id, v_{i}=b_{i}=0\in \bR^{d}, Z=\sum_{i=1}^{k}X_{i}K_{i}K^{-1}
\]
\item [Step 1] Set $\rm{loss}=\aE(\mathbf{X},\mathbf{K},\mathbf{Q},\mathbf{v})$
\item [Step 2] For $i$ in $1,\dotsc,k$
\begin{enumerate}
\item Update $b_{i}=\frac1{n_{i}}\sum_{j=1}^{n}(ZK_{i})(j)$
\item\label{s:missingK} Use SVD to solve 
\[
((Z-\mathbf{1}b_{i})K_{i}-(X_{i}-\mathbf{1}a_{i})K_{i}K^{-1})
(Q_{i}(X_{i}-\mathbf{1}a_{i})K_{i})^{T}
=U_{i}\Sigma_{i}V_{i}^{T}
\]
for $U_{i}$, $V_{i}$ orthogonal
$d\times d$ matrices and $\Sigma_{i}$ a non-negative diagonal matrices
with diagonal entries in decreasing order. 
\item Update $Q_{i}=UV^{T}$, $v_{i}=b_{i}-Q_{i}a_{i}$
\item Update $Z=\sum_{i=1}^{k}(Q_{i}X_{i}+\mathbf{1}v_{i})K_{i}K^{-1}$.
\end{enumerate}
\item[Step 3.] If $|\mathrm{loss}-\aE(\mathbf{X},\mathbf{K},\mathbf{Q},\mathbf{v})|\geq \mathrm{tol}$
and the number of iterations is less than $\rm{max\_iter}$, iterate
from Step 1. Else:
\item[Step 4.] Return $(Q_{i}),(v_{i}),Z$.
\end{enumerate}
\end{alg}

There are several obvious ways to improve the efficiency of this
algorithm. We mention only a few of the most obvious as much will
depend on the implementation of the libraries.  Clearly several values
should be stored rather than 
recomputed (including the loss functional
$\aE(\mathbf{X},\mathbf{K},\mathbf{Q},\mathbf{v})$, and the
transformed configurations 
$Q_{i}(X_{i}+\mathbf{1}v_{i})K_{i}$).  Also, we should adjust the
input so that the missing points in $X_{i}$ are represented by the
zero column (to use $X_{i}$ in place of $X_{i}K_{i}$) and is centered
on zero (to use $X_{i}$ instead of $(X_{i}-\mathbf{1}a_{i})K_{i}$).
In the update of $Z$, it is more efficient to subtract off
the old value of $Q_{i}(X_{i}+\mathbf{1}v_{i})K_{i}K^{-1}$ and add the
new value rather than take the sum as written.

\section{Behavior of the alternating least squares method}
\label{sec:ALSanal}

The purpose of this section is to describe what is known about the
theoretical behavior of the alternating least squares (ALS) method to
search for solutions to the
generalized Procrustes problem, which is the last step in our
main algorithm.  We review the results about this step needed for 
understanding the robustness of the main algorithm.

There are three natural questions we address:
\begin{enumerate}
\item When does the ALS method converge to a local
  optimum?

\item When are the local optima isolated?

\item How much does the output change when the input data
  are perturbed?
\end{enumerate}

All of these questions have reasonably satisfying answers more or less
in the literature, as we now review.  We discuss the case without
missing points for simplicity of notation, but the missing points case
works similarly.

\subsection*{Convergence of the ALS method}

Standard considerations imply that iterations of the ALS method
decrease the loss functional~\eqref{eqn:procrustes_loss_final}
monotonically.  This implies that the algorithm always converges.
However, examples can be produced where the ALS method outputs a
centroid that does not locally minimize the constrained loss functional, at least
in the case when we do not require the it to perform a minimum number
of iterations.  While such bad examples can be constructed by hand,
our experiments and the long literature on the subject bears out the
conjecture that generically the convergence is to a local minimum, and
it seems to always converge to a local minimum when required to
perform a reasonable minimum number of iterations.  We conjecture that
the bad examples form a low dimensional subspace of the space of all
possible input data when the number of points in each configuration is
larger than the dimension of the ambient Euclidean space.

As we explain below, up to a precision determined by the tolerance
setting for the algorithm, the ALS method always converges to a critical point for
the loss functional, constrained to the orbit of the original input configurations.
In practice, we can check that the resulting point is a local minimum using the
second derivative test.  We give a formula for the second derivative
of the constrained loss functional in the next subsection.

To see that the ALS method always converges to a critical point of the
constrained loss functional, we use the following notation.  After the
$s$th iteration of the main loop, we have a set of orthogonal
transformations that we denote here as $\mathbf{Q}^{[s]}$ and a
centroid for the transformed configurations that we denote here as
$Z^{[s]}$.  We let $\mathbf{X}^{[s]}$ denote the transformed
configurations, $X_{i}^{[s]}=Q_{i}^{[s]}X_{i}$.  The ALS method then
converges to
$\mathbf{X}^{[\infty]}=\mathbf{X}^{[\mathrm{max\_iter}]}$, and we
argue that for each $i$, the $d\times d$ matrix
$Z^{[\infty]}(X^{[\infty]}_{i})^{T}$ is (approximately) symmetric:
because $Q^{[s+1]}$ is the solution to the classical orthogonal
Procrustes problem for $X_{i}$, $Z^{[s]}$, we have that
\[
(Z^{[s]}-\frac1k Q_{i}^{[s]}X_{i})(Q_{i}^{[s+1]}X_{i})^{T}
\]
is symmetric (as discussed in Section~\ref{sec:procrustes}). This then
implies that 
\[
(Z^{[\infty]}-\frac1k X_{i}^{[\infty]})(X_{i}^{[\infty]})^{T}
=(Z^{[\infty]}-\frac1k Q_{i}^{[\infty]}X_{i})(Q_{i}^{[\infty]}X_{i})^{T}
\]
is (approximately) symmetric and since $AA^{T}$ is always symmetric,
we have that $Z^{[\infty]}(X^{[\infty]}_{i})^{T}$ is (approximately)
symmetric. 

The critical points $\mathbf X$ for the constrained loss
functional~\eqref{eqn:procrustes_loss_final_var} are precisely the
points where the $d\times d$ matrices $ZX_{i}^{T}$ are symmetric
for all $i$.  To explain this, we calculate the derivative of the
constrained loss functional.  At any point $\mathbf X$,
the tangent space of the orbit (in the centered first fixed formulation) is
canonically isomorphic to the product of $k-1$ copies of the tangent
space of $O(d)$, and so an element of this tangent space is specified by a choice of
$d\times d$ anti-symmetric matrix $A_{i}$ for
$i=2,\dotsc,k$. (To give uniform formulas, we set $A_{1}$ to be the
$d\times d$ zero matrix.) Integrating along such an element gives the path 
in the orbit $(Q_{i}(t)X_{i})$ where $Q_{i}(t)=e^{A_{i}t}$.
If we write $Z(t)$ for the centroid as a function of $t$, we then have
\[
\left.\frac{d}{dt}\right|_{t=0}\hspace{-1em}
\mathcal{E}(\mathbf{X},\mathbf{Q}(t))
=-2kZ \cdot \biggl(\left.\frac{d}{dt}\right|_{t=0}\hspace{-1em}
Z(t)\biggr)
=-2\sum_{i=2}^{k}\sum_{j=1}^{n} Z(j) \cdot A_{i}X_{i}(j).
\]
For the derivative to vanish, we therefore must have
\[
\sum_{j=1}^{n} Z(j) \cdot AX_{i}(j)=0
\]
for all $i=2,\dotsc,k$ and all $d\times d$ anti-symmetric matrices
$A$.  Equivalently, for every anti-symmetric bilinear form $\Phi$ on
$\bR^{d}$, we must have 
\[
\sum_{j=1}^{n} \Phi(Z(j),X_{i}(j))=0
\]
for all $i=2,\dotsc,k$, and this is equivalent to the requirement that
the $d\times d$ matrices 
\[
\sum_{j=1}^{n} Z(j) X_{i}(j)^{T}
\]
are symmetric for all $i=2,\dotsc,k$ (which also implies $\sum_{j}Z(j) X_{1}(j)^{T}$ is symmetric).

\subsection*{Isolation of the critical points}

We cannot expect numerical stability of the limit of the ALS method
unless the critical point it converges to is isolated.  Numerical experiments indicate
that the ALS method converges to a critical point with positive definite
Hessian for the loss functional; mathematically, such points are always
isolated local minima.

The Hessian can be calculated using the second derivative of the paths
considered in the previous subsection. For $Q_{i}=e^{A_{i}t}$, we get
\[
\left.\frac{d^{2}}{dt^{2}}\right|_{t=0}\hspace{-1em}
\mathcal{E}(\mathbf{X},\mathbf{Q}(t))
=-2\sum_{j=1}^{n}\biggl(\frac1k\biggl\|\sum_{i=2}^{k}A_{i}X_{i}(j)\biggr\|\raisebox{1.25ex}{\mbox{$^{2}$}}
+Z(j)\cdot\bigl(
\sum_{i=2}^{k}A_{i}^{2}X_{i}(j)\bigr)\biggr).
\]
For fixed $\mathbf{X}$, defining a quadratic form $q(\mathbf{A})$ by
the formula above, the Hessian at $\mathbf{X}$ is given by the formula
\[
H(\mathbf{A},\mathbf{B})=\tfrac12(q(\mathbf{A}+\mathbf{B})-q(\mathbf{A})-q(\mathbf{B})).
\]
Choosing a basis for $d\times d$ anti-symmetric matrices, the Hessian
becomes a symmetric $(k-1)(d-1)(d-2)/2$-dimensional square matrix.
The quadratic form $q$ is positive definite if and only if the resulting matrix
has only positive eigenvalues.

\subsection*{Perturbation of the input data}

The ALS method is stable in small perturbations of
generic input data in the following sense: for any $\epsilon>0$ and
for every point $\mathbf{X}$ in the data space with property that the
$d\times d$ matrices $Z X_{i}^{T}$ is non-singular for all $i$,
there is a neighborhood around 
$\mathbf{X}$ (whose size depends on $\mathbf{X}$) where the ALS
method converges to a point within $\epsilon$ of the limit
$\mathbf{X}^{[\mathrm{max\_iter}]}$ for $\mathbf{X}$.  This stability is a consequence of
two basic well-known stability results for the singular value
decomposition, as we now explain.

The main result of~\cite[2.1]{Soderkvist} studies stability of the
solution of the orthogonal Procrustes problem.  In our notation, it
shows that when each point cloud $X_{i}$ and the pointwise mean $Z$
are in general position, then for
$\epsilon>0$ small enough, there is a neighborhood around the $X_{i}$ where the orthonormal matrix $Q_{i}$ solving the orthogonal
Procrustes problem for $X_{i}$ and $Z$ has Frobenius norm less
than $\epsilon$.  The size of the neighborhood depends on the smallest
two singular values of the $Z X_{i}^{T}$ (which are non-zero under the
non-singular hypotheses).  We
have control over the change in singular values under perturbations by
Mirsky's theorem~\cite[\S2]{Stewart}.

Putting these together and using the fact that when actually
implemented, the ALS method has finitely many iterations,
we can conclude stability as described above.  While the theoretical
bounds give very pessimistic estimates of the size of neighborhood of
control, in practice we find experimentally that the ALS method
is stable for perturbations roughly the same order of
magnitude as $\epsilon$.  In the case when the point clouds are
all isometric up to small perturbations, the stability increases and
input perturbations of order $\epsilon$ give output perturbations of 
order $\epsilon^{2}$ (see \cite[3.1]{Sibson1979} and the proof of 
\cite[6.1]{LangronCollins}). 

\section{Stability and robustness of the algorithm}\label{sec:stability}

Under reasonable hypotheses on the input data, with high probability,
the algorithm has the following stability and robustness properties.

\subsection*{Stability in choice of subsamples}

The algorithm is designed for data expected to approximate a
contractible neighborhood of a manifold embedded in a high dimensional
Euclidean space (or a disjoint union of such).  When this is the case,
with high probability a uniformly randomly chosen set of subsamples
will have a unique good cluster, which is the restriction of a unique
good cluster for the set of all subsamples~\cite{BGMP2014}.  The
centroid of the unique cluster of the random subsamples will
approximate the centroid of the unique good cluster for all subsamples
in Procrustes distance.

\subsection*{Stability in presence of noise in data values}

If the noise is small enough that it results in only a small
distortion in the results of the dimensionality reduction algorithm, the
stability of the Procrustes algorithm discussed in
Section~\ref{sec:ALSanal} combined with the stability in choice of
subsamples implies that the resulting final embedding in the presence
of small noise will be close in the Procrustes distance to the
embedding that would have been produced in the absence of noise.

\subsection*{Robustness in presence of outliers}

Tautologically, outliers that consistently distort the embeddings for
any subsample containing them will not be in any of the subsamples in
the unique good cluster.  Thus, they will be identified as outliers by
the algorithm and omitted from the final embedding.  In some cases,
the dimensionality reduction algorithm may still return good embeddings in
the presence of a small number of outliers that are comparatively
close to the rest of the points of the subsample.  When these
embeddings are averaged with the relatively larger number of
embeddings in the good cluster that do not contain the outliers, their
overall contribution to the final embedding becomes small.
Considering the two cases, we see that even with addition of a small
number of outliers of any kind the output of the algorithm is a
low-distortion embedding of the majority of the data that is close in
Procrustes distance to the embedding that would be computed if the
outliers were omitted.

\subsection*{Robustness in bad parameter choices}

As in the case of sufficiently bad outliers, samples corresponding to
parameter choices that result in distorted embeddings that are far
from the unique good cluster will simply be discarded.  As a
consequence, the final output is generally insensitive to a small
number of isolated bad parameter choices.

\section{Synthetic experiments in manifold learning}
\label{sec:synthetic}

This section describes some experiments with synthetic data.  The
first set of experiments
(Subsections~\ref{ss:swissbegin}--\ref{ss:swissend}) validate the claim
about robustness of the algorithm via numerical experiments.  It uses
the familiar Swiss roll example.  In this example, our hypotheses
about the data set holds: namely it is a high(er) dimensional
embedding of a contractible subspace of a low dimensional manifold (or
a disjoint of such).
Our algorithm consistently constructs a good 2-dimensional embedding
even with the addition of significant noise, outliers, and parameter
variation. The second example (Subsection~\ref{ss:bucky}) explores what happens when these
hypotheses are violated, using an analysis of data lying on
``buckyballs'', which are spherical and not contractible. This data
does not admit a low-distortion embedding into $\bR^2$ and this can be
seen in the intermediate steps of our procedure. Clustering the
embeddings in Procrustes space reveals this failure and moreover
allows us to analyze the structure of the collection of embeddings.
These steps give clear indication that no good $2$-dimensional
embedding exists.  Finally, we explore the use of our algorithm to
handle data sets too large to analyze without subsampling
(Subsection~\ref{ss:large}). 

\subsection{Robustness in the presence of ambient noise concentrated around the samples}
\label{ss:swissbegin}

This subsection describes a warmup experiment that illustrates the way
that our algorithm corrects for some of the variation introduced by
subsampling.  In order to make this clearer, rather than simply
subsampling noiseless data, we add noise {\em to the subsamples}.  In
principle this could model a situation in which we do not have access
to the data except through noisy subsamples, but this is not a common
use-case.

In this experiment we simulate making multiple noisy measurements on
the same data set, which we use as our subsamples for our algorithm.  We
create a Swiss roll dataset with 2000 points.  Then, we create 200
simulated noisy measurements by subsampling and injecting additive
Gaussian noise to the subsamples.  Unpacking the steps of our
algorithm, the first step is to apply Isomap, and we see that there
are two broad classes of resulting embeddings: one in which the Swiss
roll is successfully unrolled into a sheet, and one in which the Swiss
roll remains coiled up. See
Figure~\ref{fig:noisy_swissroll_samples}. (This is what is expected;
e.g., see~\cite{Balasubramanian2002}.)

     \begin{figure}
         \centering
         \includegraphics[width=0.7\textwidth]{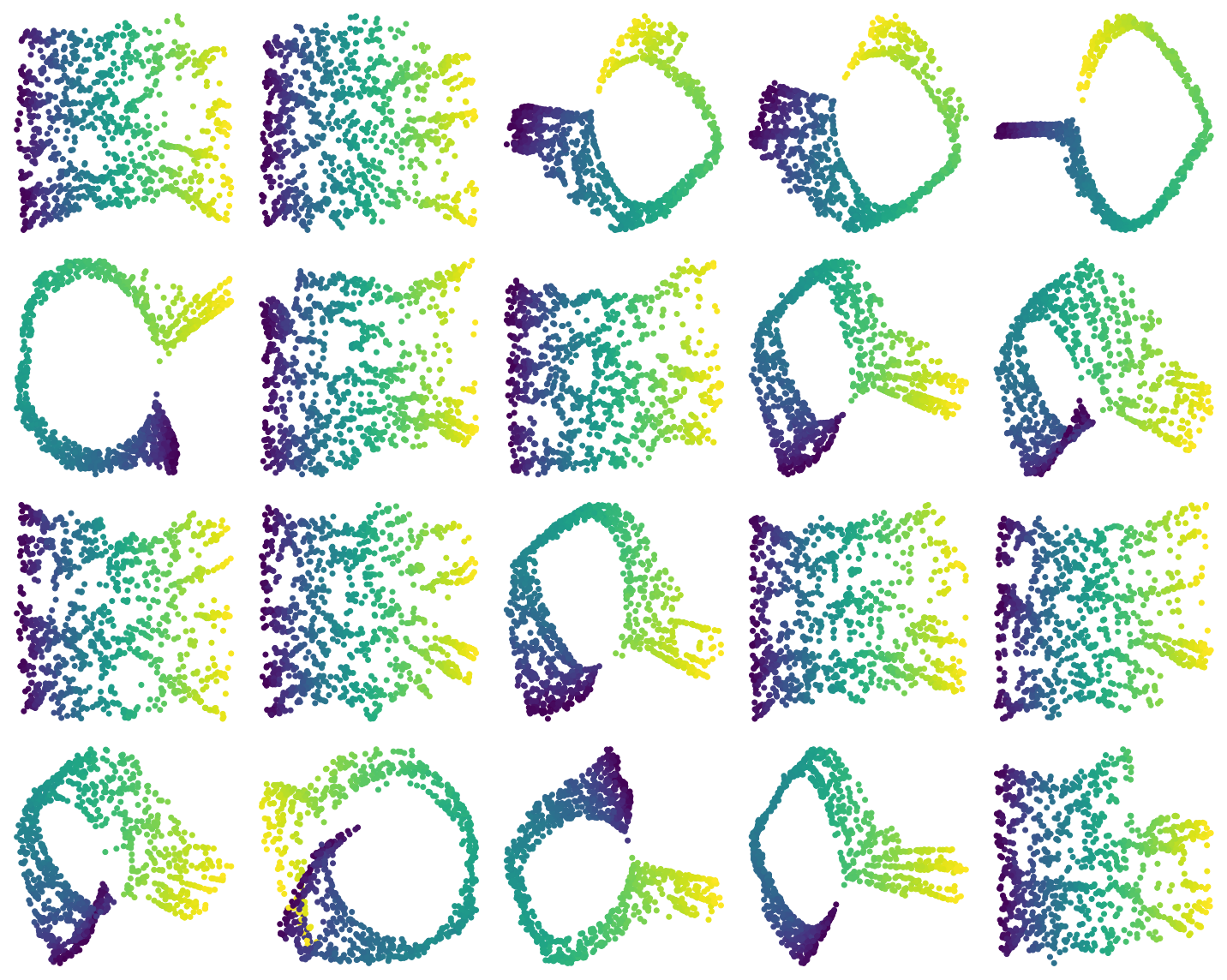}
         \caption{Some example outputs of Isomap applied to a noisy Swiss roll dataset.}
         \label{fig:noisy_swissroll_samples}
     \end{figure}

The next step in the algorithm is to calculate Procrustes distances
between these embeddings.  We have visualized the resulting graph
using multidimensional scaling MDS) in
Figure~\ref{fig:noisy_swissroll_mds}.  The two types of embeddings
roughly divide into two major clusters, one consisting of unrolled
sheets and one consisting of coiled up sheets.  The appearance of the
two clusters indicate the lack of robustness in applying Isomap to the
Swiss roll with the current level of added noise.  We also observe
that while most unrolled outputs are close together, the coiled
versions can differ greatly among themselves. (The unrolled outputs
form a dense cluster while the coiled outputs do not.) The next step in the
algorithm is to calculate the persistent homology of each embedding to
identify the cluster consisting of the embeddings that have only very small
(noise) loops in $PH_1$.  This is precisely the cluster of the
unrolled outputs.

\begin{figure}
         \centering
         \includegraphics[width=0.7\textwidth]{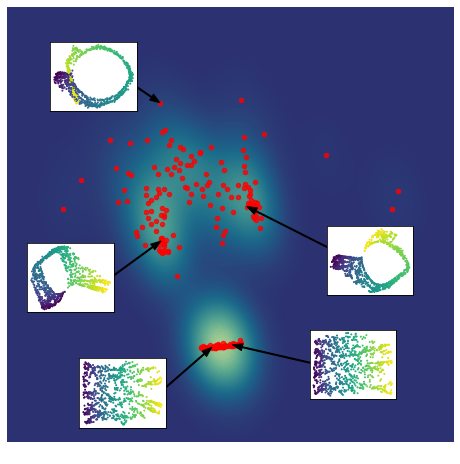}
         \caption{An MDS representation of the Procrustes distances between Isomap outputs of 200 noisy Swiss rolls: each point represents an output, and the plot is shaded according to density of the points.  Some representative points are indicated together with their corresponding Isomap embedding.}
         \label{fig:noisy_swissroll_mds}
     \end{figure}

The final step is to align and average the embeddings in the
unrolled cluster.  The final averaged embedding is illustrated in Figure~\ref{fig:noisy_swissroll_average}.

     \begin{figure}
         \centering
         \includegraphics[width=0.4\textwidth]{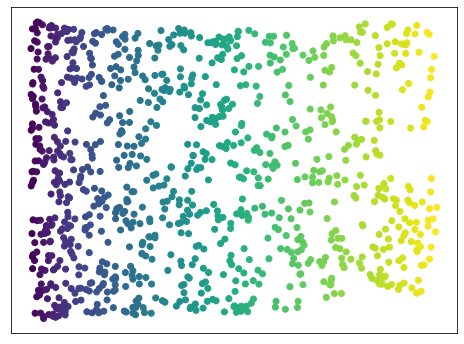}
         \caption{The Procrustes-aligned Isomap embedding corresponding to the cluster of unrolled Swiss rolls.}
         \label{fig:noisy_swissroll_average}
     \end{figure}

\subsection{Robustness in presence of ambient noise concentrated around the manifold}

In this experiment we simulate sampling from a data set corrupted by
ambient noise.  We again create a Swiss roll dataset with 2000 points.
Then, we inject additive Gaussian noise to this dataset, so that
Isomap produces a corrupted coiled output when run on the whole data
set.  See Figure~\ref{fig:Gaussiancoiledisomap}.  When we randomly 
subsample this dataset to obtain 500 samples of 1000 points each,
for almost all of these samples the result of Isomap is a coiled
embedding; only a handful are unrolled into a sheet.  See
Figure~\ref{fig:Gaussiancoiledsubsamples}.

The next step in the algorithm is to calculate Procrustes distances
between these embeddings and cluster.  In this case, there are many
clusters, but calculating the persistent homology of each embedding
identifies a dense cluster consisting of the embeddings
that have only very small (noise) loops in $PH_1$, in contrast to most of
the clusters which are diffuse and have significant loops in $PH_1$.
See Figure~\ref{fig:Gaussiancharts}.

Finally, we align and average the embeddings in the unrolled cluster.
The final averaged embedding is illustrated in
Figure~\ref{fig:Gaussianaverage}.

\begin{figure}
\begin{subfigure}[t]{0.45\textwidth}
\centering
\includegraphics[width=0.8\textwidth]{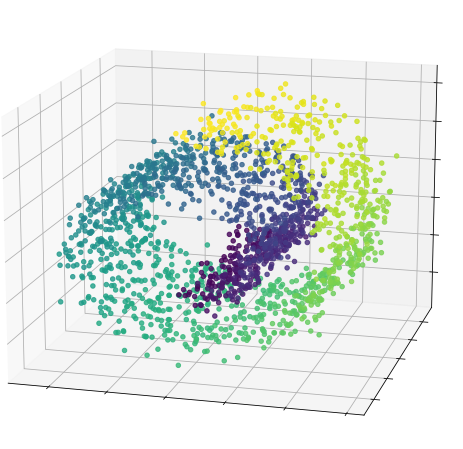}
\caption{The noisy swiss roll.}
\end{subfigure}
\begin{subfigure}[t]{0.45\textwidth}
\centering
\includegraphics[width=0.8\textwidth]{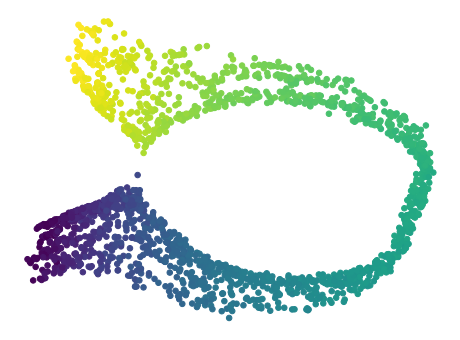}
\caption{Isomap does not return the correct embedding.}
\end{subfigure}
\caption{Additive Gaussian noise in the ambient space.}\label{fig:Gaussiancoiledisomap}
\end{figure}

\begin{figure}
\centering
\includegraphics[width=0.8\textwidth]{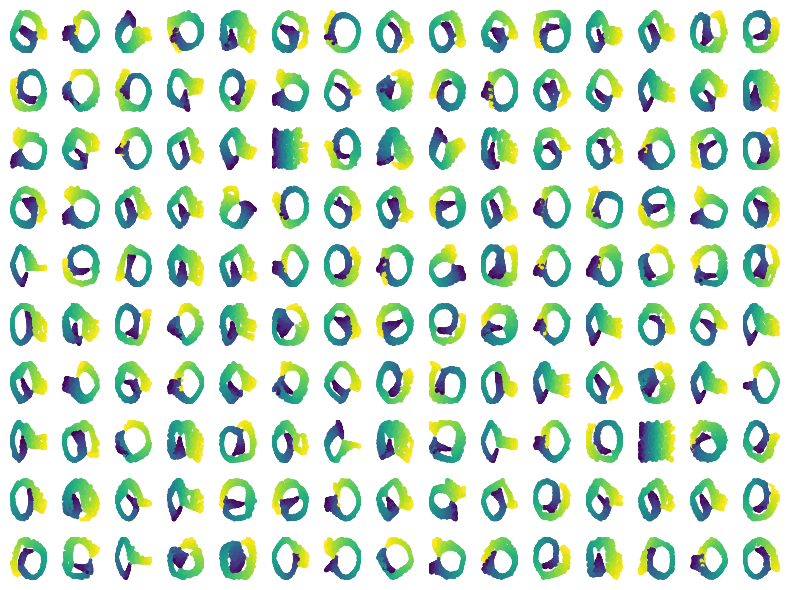}
\caption{All but a few of the subsamples result in coiled Isomap embeddings.}\label{fig:Gaussiancoiledsubsamples}
\end{figure}

\begin{figure}
\begin{subfigure}[t]{0.45\textwidth}
\centering
\includegraphics[width=0.8\textwidth]{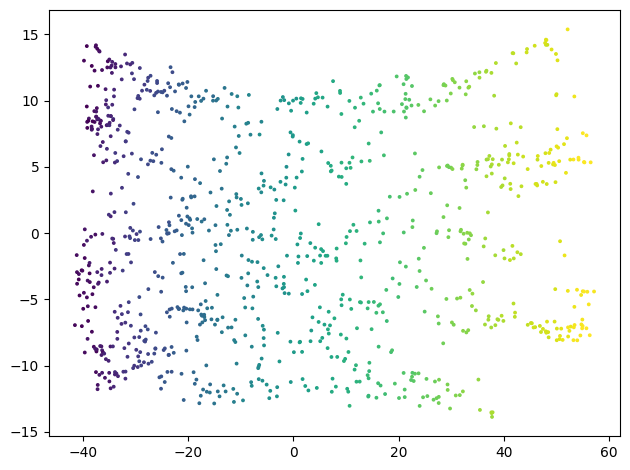}
\caption{One of the few subsamples resulting in an unrolled coordinate chart.}
\end{subfigure}
\begin{subfigure}[t]{0.45\textwidth}
\centering
\includegraphics[width=0.8\textwidth]{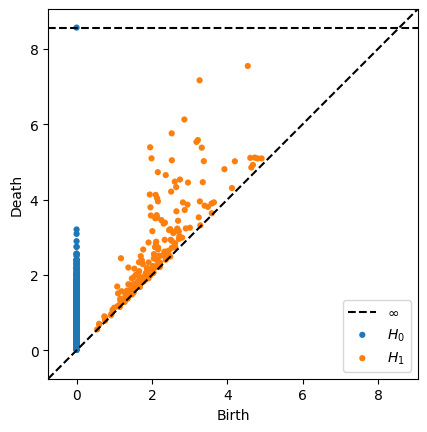}
\caption{The persistence diagram has a single large class in $PH_0$
  and no large classes in $PH_1$.}
\end{subfigure} \\

\begin{subfigure}[t]{0.45\textwidth}
\centering
\includegraphics[width=0.8\textwidth]{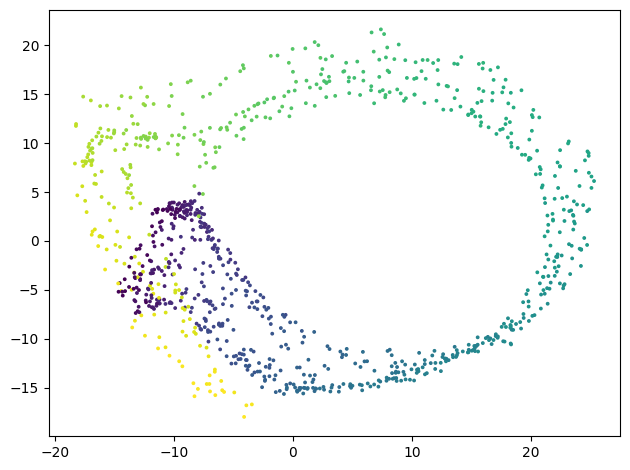}
\caption{A generic subsample resulting in a coiled coordinate chart.}
\end{subfigure}
\begin{subfigure}[t]{0.45\textwidth}
\centering
\includegraphics[width=0.8\textwidth]{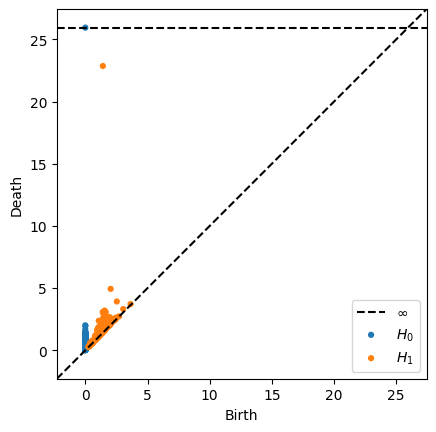}
\caption{The persistence diagram has a single large class in $PH_0$
  but also has a very large class in $PH_1$ capturing the loop.}
\end{subfigure}
\caption{Comparing clusters in Procrustes space using persistent
  homology.}
\label{fig:Gaussiancharts}
\end{figure}

\begin{figure}
\centering
\includegraphics[width=0.8\textwidth]{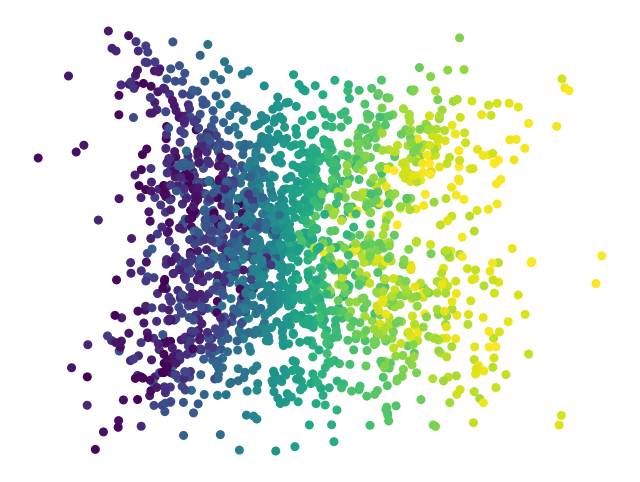}
\caption{The Procrustes-aligned Isomap embedding corresponding to the cluster of unrolled Swiss rolls.}\label{fig:Gaussianaverage}
\end{figure}

\subsection{Robustness in presence of outliers}

In these experiments, we imagine a single dataset of measurements
where one or more bad sample points (``outliers'') are included.  We
work again with the Swiss roll.  Even a single adversarial outlier can
cause a short-circuit in the inferred connectivity of the manifold,
and the resulting Isomap embedding will fail to unroll the Swiss
roll. See Figure~\ref{fig:outlier}, parts~\subref{figs:oa}
and~\subref{figs:ob}.  When we randomly subsample this dataset to obtain 200
samples of 600 points each, most of these samples do not contain the
outlier, and in these cases Isomap manages to unroll the Swiss
roll. See Figure~\ref{fig:outlier}, part~\subref{figs:oc}.  In some
instances, Isomap unrolls the roll even if the outlier is included in
the subsample if the other non-outlier points involved in the
short-circuiting are omitted.  Finally, we align all the outputs in
the good cluster and compute the centroid: the result is a
successfully unrolled Swiss roll. See Figure~\ref{fig:outlier},
part~\subref{figs:od}.

    \begin{figure}
        \centering
        \begin{subfigure}[t]{0.45\textwidth}
            \centering
            \includegraphics[height=1.4in]{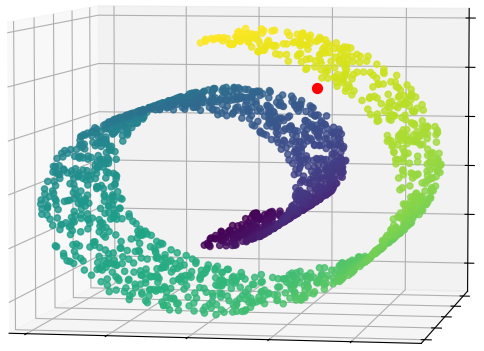}
            \caption{Original data in 3D}\label{figs:oa}
        \end{subfigure}%
        \begin{subfigure}[t]{0.45\textwidth}
            \centering
            \includegraphics[height=1.4in]{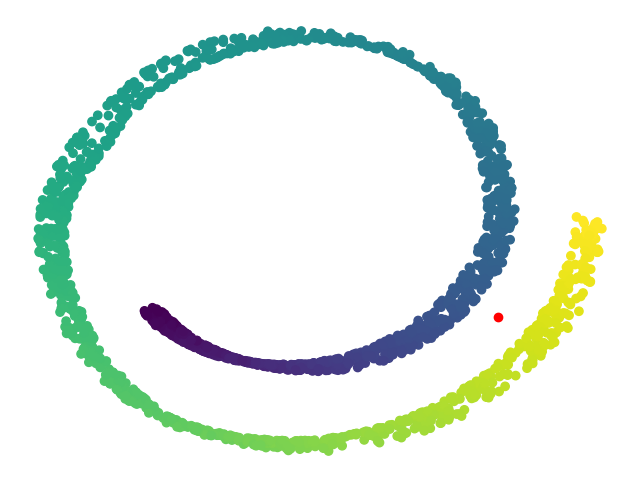}
            \caption{Isomap result}\label{figs:ob}
        \end{subfigure}
        \\
        \begin{subfigure}[t]{0.45\textwidth}
            \centering
            \includegraphics[height=1.4in]{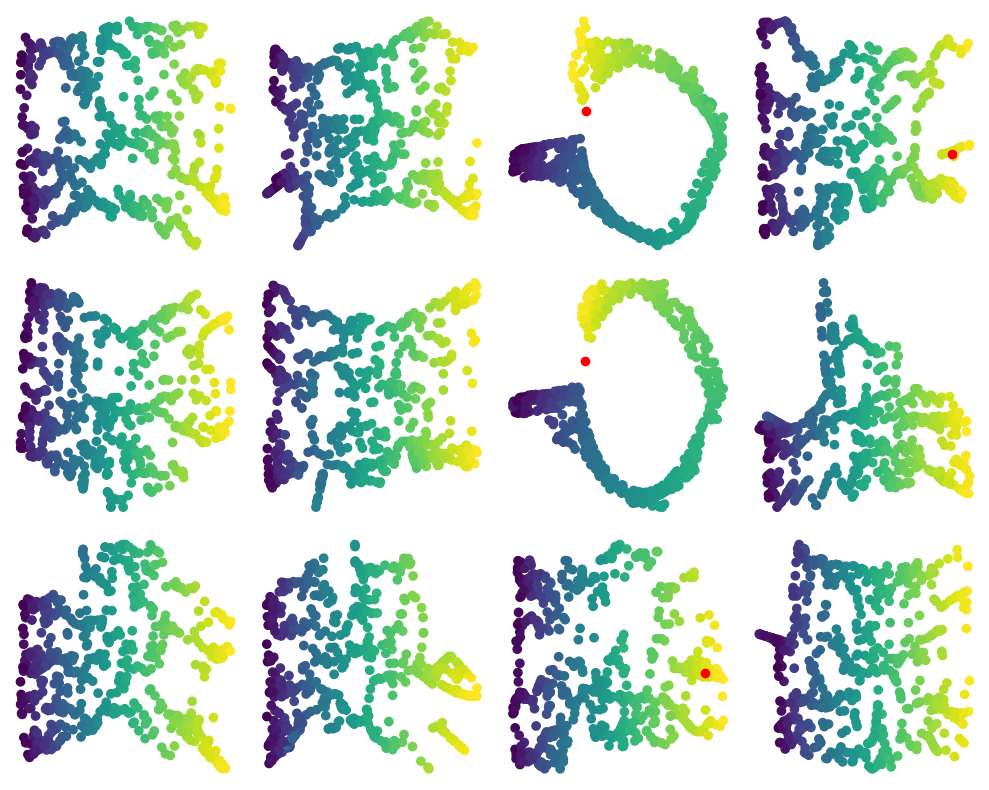}
            \caption{Isomap applied to subsamples}\label{figs:oc}
        \end{subfigure}%
        \begin{subfigure}[t]{0.45\textwidth}
            \centering
            \includegraphics[height=1.4in]{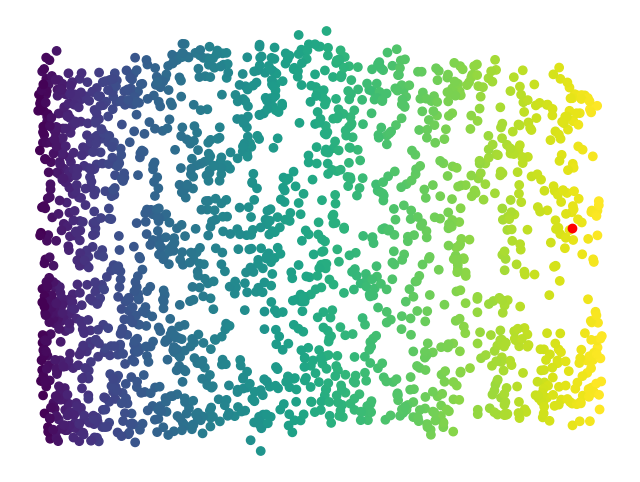}
            \caption{Final aligned output}\label{figs:od}
        \end{subfigure}
        \caption{Robustifying the output of Isomap (single outlier case)}
        \label{fig:outlier}
    \end{figure}

The previous experiment illustrates the need for some kind of
robustification (in this case by taking subsamples) even in the
presence of a single outlier.  The next experiment studies the more
expected case where there are multiple outliers.  In this experiment,
we add 100 (5\%) outliers chosen uniformly from a bounding box.  As in
the single outlier case, we subsampled the dataset 200 times to obtain
samples of size 600 and ran Isomap on these subsamples.  With this
many outliers, the coiled outputs are in the majority.  We clustered
the outputs and identified a cluster of uncoiled outputs using
persistent homology.  Then, we computed the Procrustes alignment and
averaged over the subset of these good outputs, producing the expected
unrolled Swiss roll.  See Figure~\ref{fig:outliers}.

    \begin{figure}
        \centering
        \begin{subfigure}[t]{0.45\textwidth}
            \centering
            \includegraphics[height=1.4in]{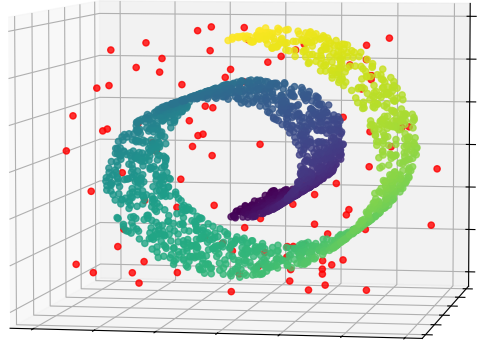}
            \caption{Original data in 3D}
        \end{subfigure}%
        \begin{subfigure}[t]{0.45\textwidth}
            \centering
            \includegraphics[height=1.4in]{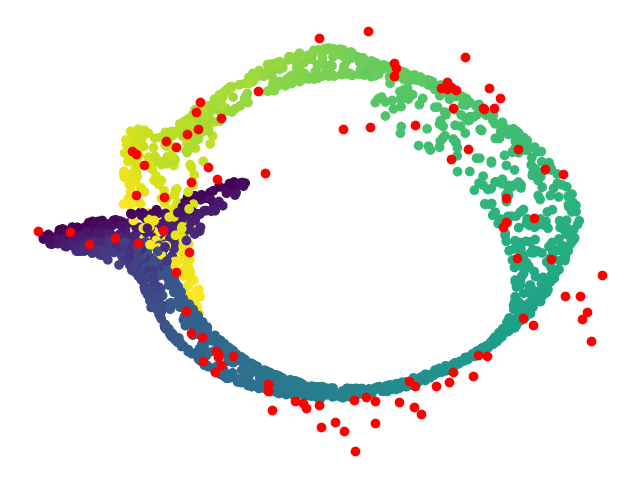}
            \caption{Isomap result}
        \end{subfigure}
        \\
        \begin{subfigure}[t]{0.45\textwidth}
            \centering
            \includegraphics[height=1.4in]{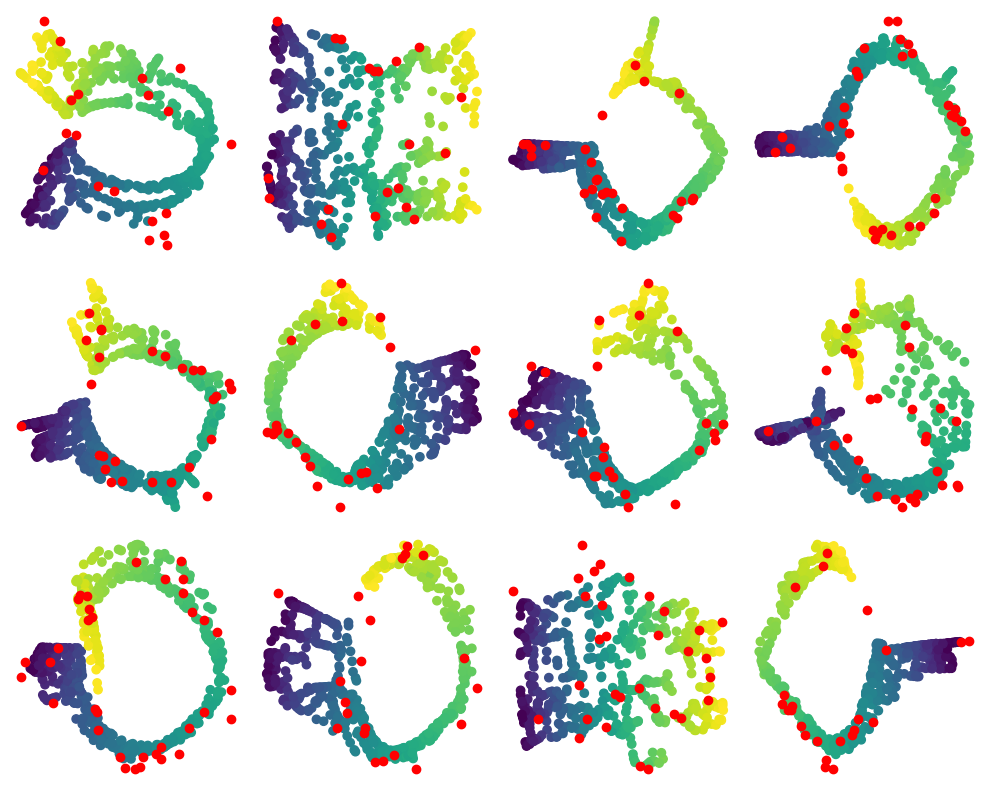}
            \caption{Isomap applied to subsamples}
        \end{subfigure}%
        \begin{subfigure}[t]{0.45\textwidth}
            \centering
            \includegraphics[height=1.4in]{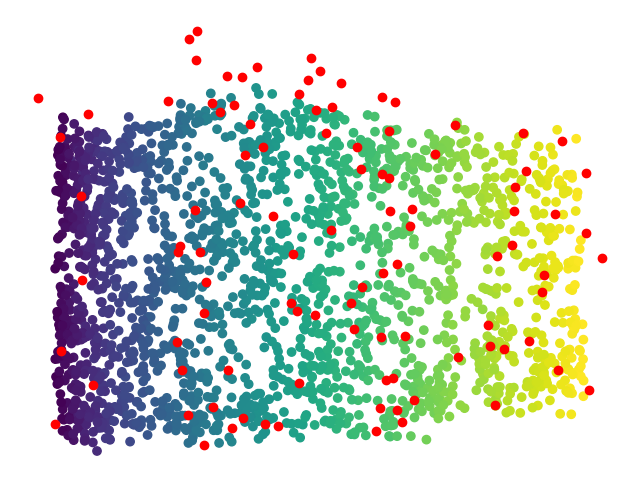}
            \caption{Final aligned output}
        \end{subfigure}
        \caption{Many outliers}
        \label{fig:outliers}
    \end{figure}

These experiments demonstrate that our procedure is robust to
outliers, in the sense that the output in the presence of a small
percentage of outliers is
a low-distortion unrolled embedding of the type we expect from Isomap
applied to noiseless samples from the underlying manifold.

\subsection{Parameter variation}
\label{ss:swissend}

All manifold learning algorithms have hyperparameters that have to be
chosen by the practitioner.  In the case of Isomap, there is a single
parameter corresponding to neighborhood size that needs to be
judiciously chosen in order to obtain good results.  The next
experiment illustrates that under the hypothesis that the data represents
a convex subset of a low dimensional manifold embedded in a higher
dimensional manifold, we can detect a cluster of good parameters using
our persistent homology methodology. 

In this experiment, we apply Isomap to the 2000 point Swiss roll
synthetic data set under a wide range of parameters. We have graphed
the heatmap of Procrustes distances of the resulting embeddings 
in Figure~\ref{fig:parameters}. We see that at
small radii, Isomap fails to find a good embedding because the points
are not connected (these have high rank large bars in $PH_{0}$),
whereas at large radii, the embeddings 
``short circuit'' in the normal direction of the manifold, leading to
coiled representations (these have a large bar in $PH_{1}$).  Between these extremes, there is a
range of values that lead to good embeddings (with no large bars in $PH_{1}$).  Moreover, these
embeddings are close together in Procrustes distances.  Within this range,
we recover the expected unrolled embedding of the Swiss roll dataset.
Furthermore, there is a sharp transition between unrolled and coiled
up representations, as shown by the large Procrustes distance between
each embedding in the unrolled cluster and each embedding in the
coiled cluster. The figure illustrates four clusters in total: the first
box shows the cluster of small radii; the next 19 boxes show the
cluster of unrolled embeddings; the next single box shows a
transitional embedding far from both  the unrolled embeddings and the
coiled embeddings; and finally, the last nine boxes show the cluster
of coiled embeddings.  A more careful analysis of the distances (which
we omit here) reveals that they correspond to the size of the gap between
sheets of the roll; this is closely related to the ``condition
number'' and reach of the manifold~\cite{Niyogi2008}.

    \begin{figure}
        \centering
        \includegraphics[width=0.8\textwidth]{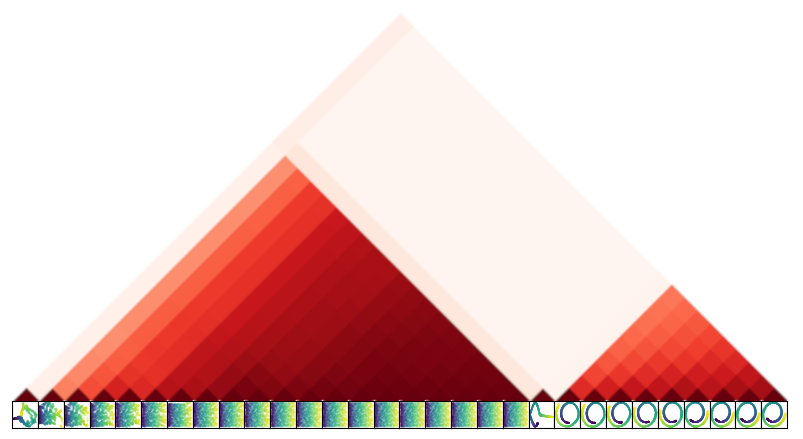}
        \caption{The Procrustes distance matrix between Isomap outputs
	for varying neighborhood size parameters, increasing from left
	to right.  The output embeddings are shown along the bottom,
	and the triangle heatmap shows the distances: the darker the
	color, the smaller the Procrustes distance. (To compare two
	outputs: intersect the $45^{\circ}$ lines emanating from them to
	find the box on the heatmap illustrating their Procrustes distance.)} 
        \label{fig:parameters}
    \end{figure}

\subsection{Detecting failure of embedding methodology}
\label{ss:bucky}

The next experiment studies a case contrary to the general hypothesis
that the data comes from a convex subset of a low dimensional
manifold.  In this experiment, we use the vertices of
buckminsterfullerene or ``buckyballs''
for the basic data .  A buckyball is a
molecule consisting of 60 carbon atoms arranged on a sphere.  Its
bonds forms hexagonal and pentagonal rings resembling the surface of a
soccer ball.  See Figure~\ref{fig:buckyball}, part~\subref{figs:bba}. As
a proxy for the sphere $S^2$, the positions of the 
atoms of a buckyball in 3D cannot be faithfully compressed down to two
dimensions, unlike the unrolling of the Swiss roll.  

We computed the Isomap outputs of 500 noisy buckyballs, and the
Procrustes distances between them.  See Figure~\ref{fig:buckyball},
part~\subref{figs:bbb}.  We visualized these distances using MDS in
Figure~\ref{fig:buckyball}, part~\subref{figs:bbc}; however, in this case
the 2-dimensional rendering obscures rather than illuminates the
geometric structure of the Procrustes graph of Isomap outputs  (as we explain
below). Nevertheless, it illustrates enough to show that unlike in the
Swiss roll studies, there are no dense clusters of embeddings in this
experiment, and so we cannot expect an averaging procedure to yield a
reasonable result.  

    \begin{figure}
        \centering
        \begin{subfigure}[t]{0.3\textwidth}
            \centering
            \includegraphics[height=1.4in]{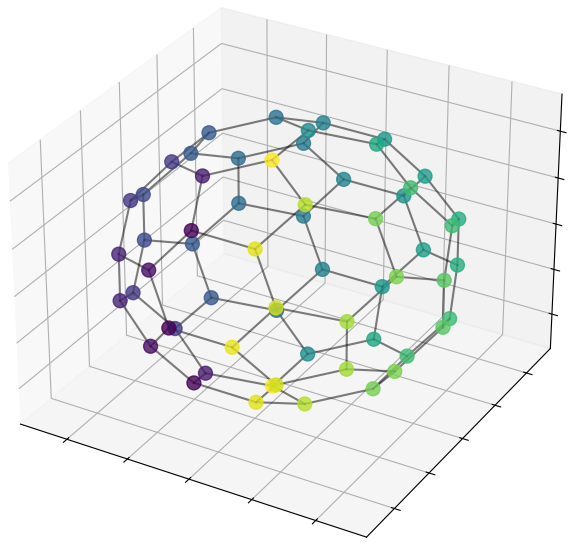}
            \caption{Original data in 3D}\label{figs:bba}
        \end{subfigure}%
        \begin{subfigure}[t]{0.3\textwidth}
            \centering
            \includegraphics[height=1.4in]{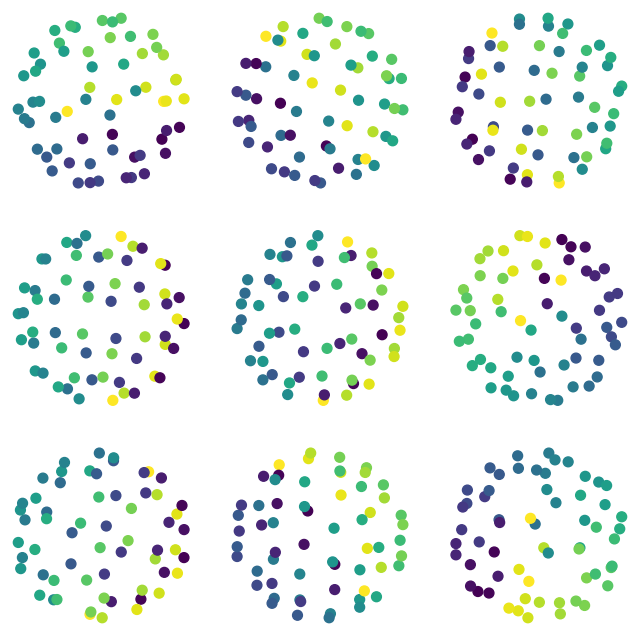}
            \caption{Isomap of subsamples}\label{figs:bbb}
        \end{subfigure}%
        \begin{subfigure}[t]{0.3\textwidth}
            \centering
            \includegraphics[height=1.4in]{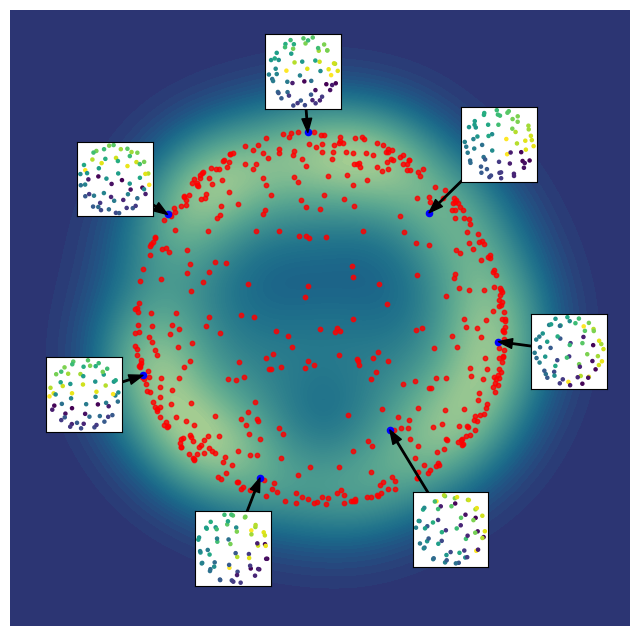}
            \caption{MDS embedding}\label{figs:bbc}
        \end{subfigure}
        \caption{Buckyball dataset}
        \label{fig:buckyball}
    \end{figure}

Although not directly related to our algorithm, we can say more about 
the geometric structure of the Isomap outputs.  We note that
the typical output is approximately a \emph{flattened
sphere}: the dimensionality reduction algorithm
is approximately linearly projecting the buckyball one
dimension down.  Thus, after centering the buckyball and the resulting
Isomap output, we are led to consider the space $V_2(\mathbb{R}^3)$ of
(orthonormal) $2$-frames in $\mathbb{R}^3$, corresponding to the basis
vectors in $\mathbb{R}^3$ that span the plane we are projecting to.
Procrustes alignment quotients out the isometries $O(2)$ of
this plane, and the Procrustes graph of outputs of Isomap on the buckyball
dataset should approximate
$V_2(\mathbb{R}^3) / O(2) \cong \mathbb{R}P^2$, the real projective
plane.  

We can test the conclusion of this thought experiment using 
persistent homology. For the projective plane, using coefficient
field $\mathbb{F}_p$, $p$ prime, 
    \begin{equation*}
        H_1(\mathbb{R}P^2; \mathbb{F}_p) \cong H_2(\mathbb{R}P^2; \mathbb{F}_p) \cong \begin{cases} \mathbb{F}_2, & \text{if } p = 2 \\ 0, & \text{if } p > 2. \end{cases}
    \end{equation*} 
Computing the persistent homology for the point cloud consisting of
the 500 Isomap outputs with the Procrustes metric with Ripser
\cite{Tralie2018}, we find persistent classes in
$PH_1(\mathbb{R}P^2; \mathbb{F}_2)$ and $PH_2(\mathbb{R}P^2;
\mathbb{F}_2)$, but not in $PH_*(\mathbb{R}P^2; \mathbb{F}_p)$ for $p
= 3$ or $5$.  See Figure~\ref{fig:buckyball-ph}.
In other words, the persistent homology signature of
the Isomap outputs matches the prediction that the landscape of Isomap
outputs at least homotopically is an approximation of 
$\mathbb{R}P^2$.

    \begin{figure}
        \centering
        \includegraphics[height=1.5in]{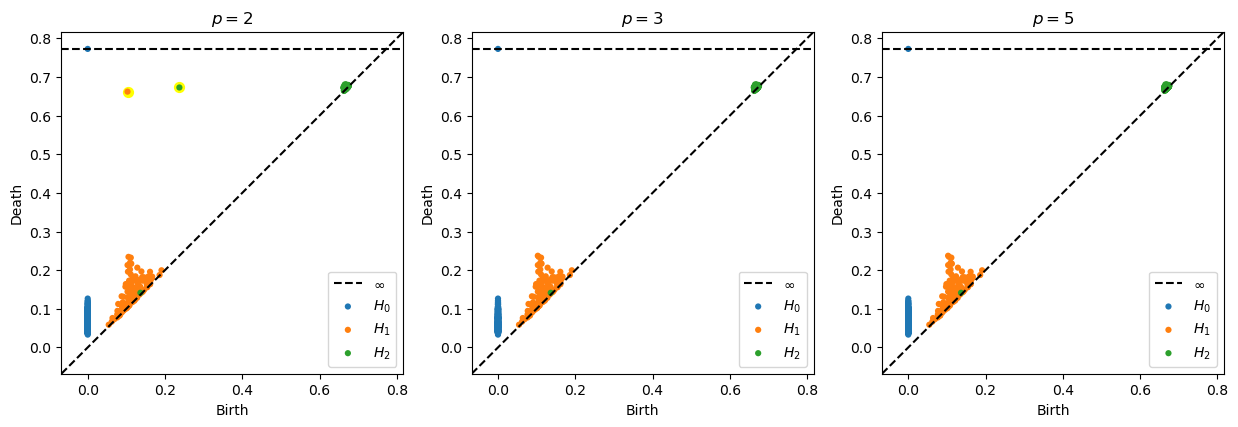}
        \caption{Persistence diagrams for the Procrustes
	graph of Isomap outputs for the buckyball dataset}\label{fig:buckyball-ph}
    \end{figure}

\subsection{Divide and conquer for very large data sets}
\label{ss:large}

Manifold learning algorithms tend to have superlinear time complexity.
This suggests that using a divide-and-conquer approach on large
datasets may result in time savings.  The idea is the same as before:
subsample the dataset, run the algorithm on the individual subsamples,
and then use the Procrustes alignment to merge the results.  A
schematic of this process is shown in Figure
\ref{fig:divide_conquer}. 
    
\begin{figure}
    \centering
    \includegraphics[width=0.8\textwidth]{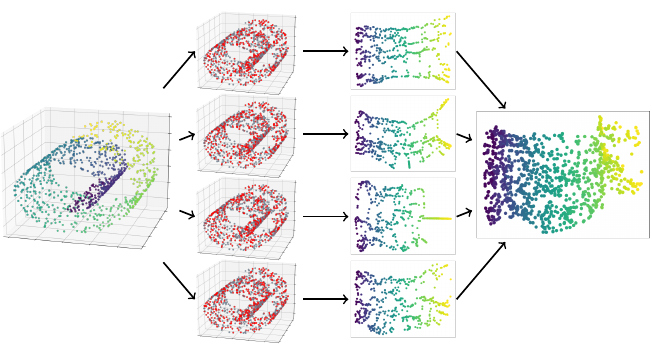}
    \caption{Starting with a large dataset, repeatedly subsample the dataset to obtain multiple subsamples of smaller size (here the subsamples are indicated in red).  Then, apply the manifold learning algorithm to each subsample and use the Procrustes alignment to average the resulting embeddings.}
    \label{fig:divide_conquer}
\end{figure}

\section{Analysis of real data coming from single-cell genomics}
\label{sec:real}

The purpose of this section is to explore an application of our
algorithm to real data.  We use two data sets as examples where
coordinates in dimensionality-reduced genomic space allows us to
determine cell types.  The first such data set comes from blood cells
(PBMC) and the second from mouse neural tissue (Tabula muris). In
these applications, in addition to Isomap, we use four other manifold
learning (dimensionality reduction) algorithms: PCA, Laplacian
eigenmaps, t-SNE, and UMAP.  In each case, we look at coordinates
resulting from applying just the dimensionality reduction algorithm
(``base case'') and our algorithm (``robustified case'').

We see interesting new phenomena in these applications.  For one
thing, using the manifold learning algorithms PCA, Laplacian
eigenmaps, and Isomap with our algorithm, we see a single large tight
cluster in the Procrustes graph of the subsamples.  However, unlike
the charts for the synthetic Swiss roll examples, both the good and
the bad embeddings have trivial $PH_{1}$, our indicator for
contractibility.  Instead, what distinguishes the good from the bad
embeddings is how evenly spread they are.  We used embeddings in
$\bR^{2}$ and the bad embeddings are ones concentrated along a single
axis, being approximately one-dimensional, whereas the good embeddings
have extent in both dimensions.  For both theoretical and practical
reasons (as explained in the discussion of the algorithm), the
embeddings that are essentially one-dimensional need to be discarded.

On the other hand, using t-SNE and UMAP, the Procrustes graph of the
subsamples are comparatively very spread out and form a single large
but extremely spread out cluster, indicating that the averaging
procedure in our algorithm cannot be expected to output reasonable
coordinates.  This dichotomy is not surprising since the PCA,
Laplacian eigenmaps, and Isomap algorithms strive to perform global
alignment of the local coordinates, whereas the t-SNE and UMAP
algorithms are designed to separate local clusters.  (This dichotomy
of behavior is a topic of interest in the genomics methodology
community, see for example~\cite{Pachter2023} for a critical review
of the use of these kinds of dimensionality reduction to produce
pictures.)

\subsection{3k PBMCs scRNA-sequencing}

The first dataset is a preprocessed version of single cell
RNA-sequencing data from 3k peripheral blood mononuclear cells (PBMCs)
from a healthy donor, made freely available by 10x Genomics.
Comprised of cells from a variety of myeloid and lymphoid lineages,
the dataset has rich cluster and continuous structure, making it a
useful dataset to benchmark new bioinformatic methods.   It can be
accessed using the \texttt{scanpy} Python package~\cite{Wolf2018} via
the following command: 
\begin{quote}
\href{https://scanpy.readthedocs.io/en/stable/generated/scanpy.datasets.pbmc3k_processed.html}{\texttt{scanpy.datasets.pbmc3k\_processed}} 
\end{quote}

As is typical in the visualization of scRNA-sequencing data, PCA is
used as a preliminary step prior to constructing the 2D embeddings: we
reduce the raw data to the first 50 principal components.  For the
base case, we then applied each of the manifold learning algorithms to
reduce the coordinates in $\bR^{50}$ to $\bR^{2}$.  For the
robustified case, we take 50 subsamples of 500 cells, and proceed with
our algorithm with each of the manifold learning algorithms used for
dimensionality reduction from $\bR^{50}$ to $\bR^{2}$.

Figure~\ref{fig:pbmc_procrustes_hist} shows the MDS plots of the Procrustes
distances between the subsample embeddings along with histograms of those values under
each of the five manifold learning algorithms.  We can see that PCA,
Laplacian eigenmaps, Isomap have much more concentrated Procrustes
distances.  And when we look at the MDS plot of the Procrustes
distances, we see there are tight clusters of the embeddings, whereas
the embeddings for t-SNE and UMAP are significantly more spread out (with a
diameter over 4 times as large).  The algorithm then proceeds for
PCA, Laplacian eigenmaps, Isomap and terminates with an error of no
remaining clusters for
t-SNE and UMAP.

For PCA, Laplacian eigenmaps, Isomap we find that the main cluster has
embeddings that are essentially two-dimensional;  the
outlying points are the embeddings that are essentially one-dimensional.  In
Figures~\ref{fig:pbmc_PCA_subsamples} and
~\ref{fig:pbmc_Isomap_subsamples} we plot the subsample embeddings for
PCA and Isomap.  (The relatively larger number of bad embeddings for
PCA in Figure~\ref{fig:pbmc_PCA_subsamples} versus Isomap
in~\ref{fig:pbmc_Isomap_subsamples} is consistent with the larger
number of embeddings outside the tight cluster as pictured in Figure~\ref{fig:pbmc_procrustes_hist}.)
We select the good cluster and apply averaging procedure
to the subsample embeddings it contains.  The output is robustified
coordinates for PCA, Laplacian eigenmaps, and Isomap.  This is
illustrated in Figure~\ref{fig:pbmc_embeddings} together with the
embeddings obtained on the entire data set.  In the figure, the data
points are colored by cell type.

\ifmam
\else
\begin{figure}
\centering
\begin{subfigure}[t]{0.45\textwidth}
\centering
	\includegraphics[height=0.75in]{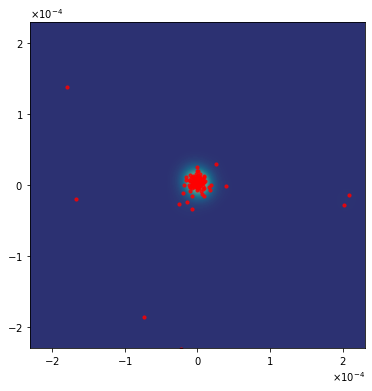}
            \includegraphics[height=0.75in]{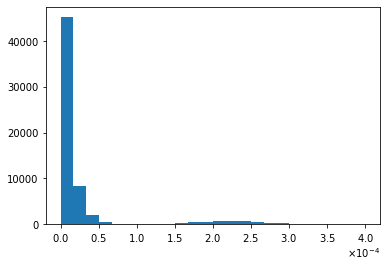}
            \caption{PCA}
        \end{subfigure}%
        \begin{subfigure}[t]{0.45\textwidth}
            \centering
            \includegraphics[height=0.75in]{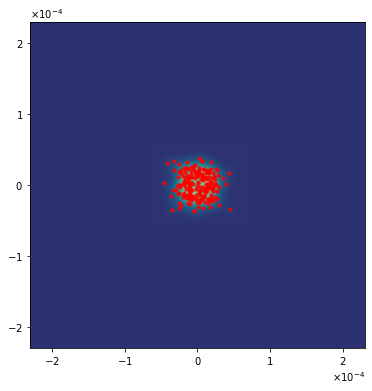}
            \includegraphics[height=0.75in]{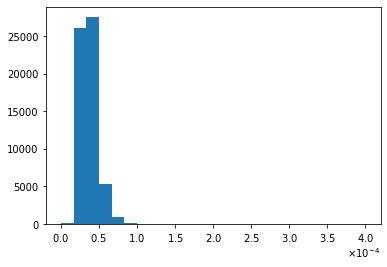}
            \caption{Isomap}
        \end{subfigure}
\\[1ex]
                \begin{subfigure}[t]{0.45\textwidth}
            \centering
            \includegraphics[height=0.75in]{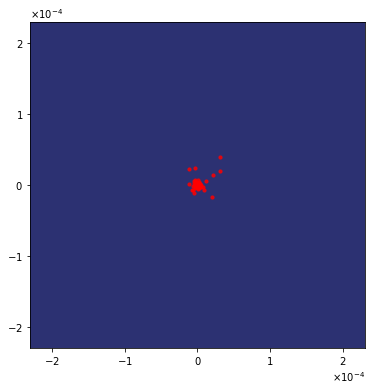}
            \includegraphics[height=0.75in]{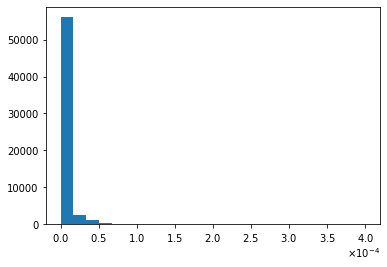}
            \caption{Laplacian eigenmaps.}
        \end{subfigure}%
        \begin{subfigure}[t]{0.45\textwidth}
            \centering
            \includegraphics[height=0.75in]{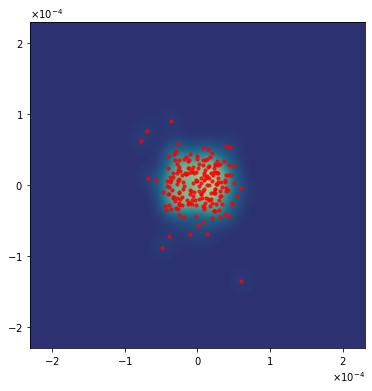}
            \includegraphics[height=0.75in]{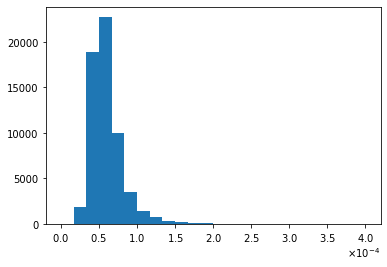}
            \caption{t-SNE.}
        \end{subfigure}%
\\[1ex]
        \begin{subfigure}[t]{0.45\textwidth}
            \centering
            \includegraphics[height=0.75in]{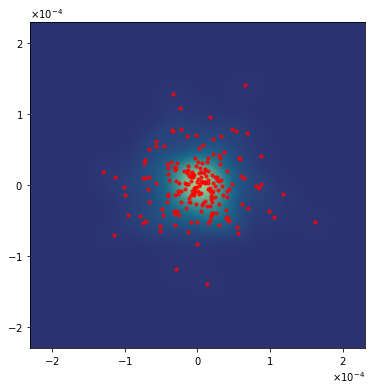}
            \includegraphics[height=0.75in]{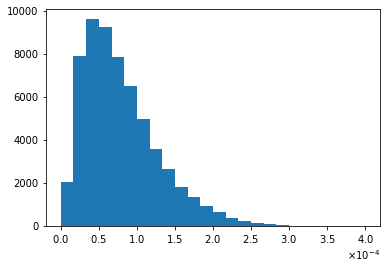}
            \caption{UMAP.}
        \end{subfigure}%

\caption{Procrustes distances between subsample embeddings for the PBMC dataset.}
\label{fig:pbmc_procrustes_hist}
\end{figure}
\fi

\begin{figure}
\centering
\includegraphics[height=3in]{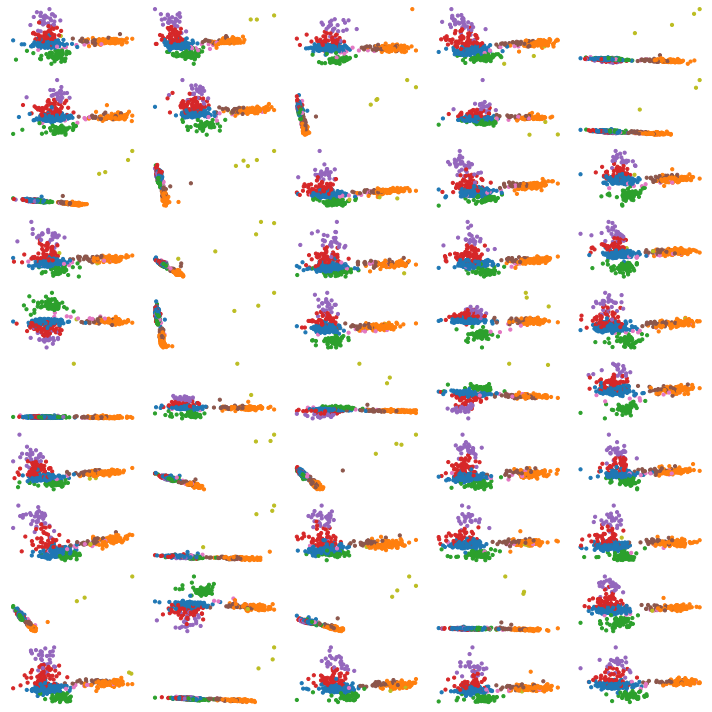}
\caption{Subsample embeddings from PCA.}\label{fig:pbmc_PCA_subsamples}
\end{figure}

\begin{figure}
\centering
\includegraphics[height=3in]{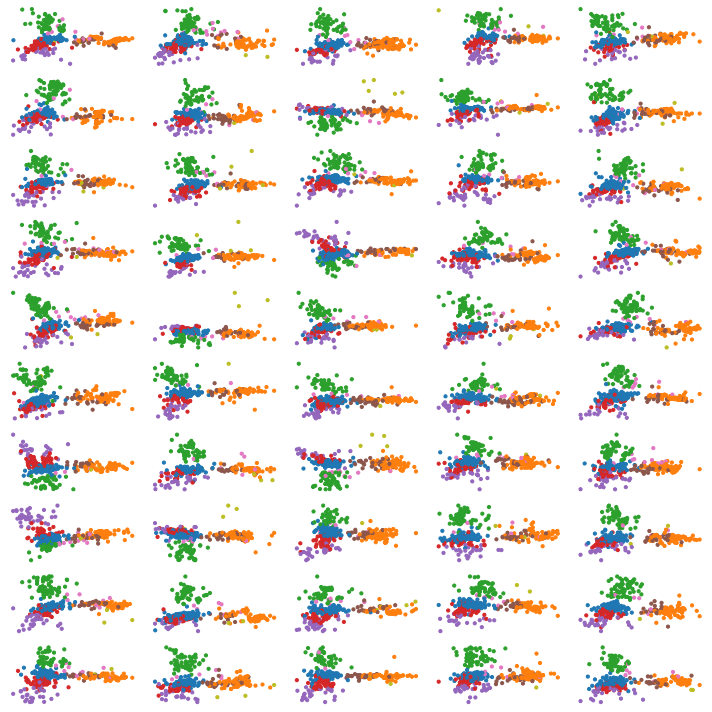}
\caption{Subsample embeddings from Isomap.}\label{fig:pbmc_Isomap_subsamples}
\end{figure}

\begin{figure}
\begin{subfigure}[t]{0.25\textwidth}
\centering
\includegraphics[height=1.5in]{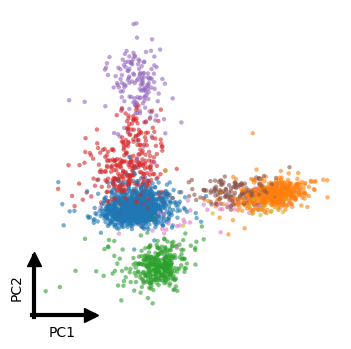}
\end{subfigure}%
\begin{subfigure}[t]{0.25\textwidth}
\centering
\includegraphics[height=1.5in]{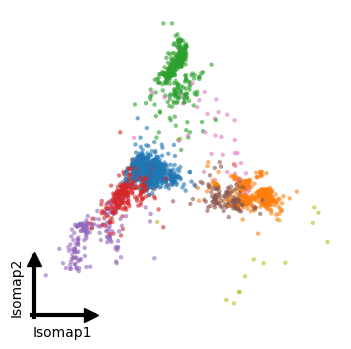}
\end{subfigure}
\begin{subfigure}[t]{0.25\textwidth}
\centering
\includegraphics[height=1.5in]{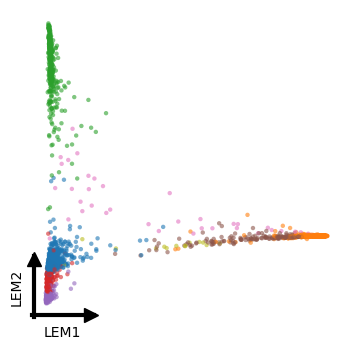}
\end{subfigure}
\\
\begin{subfigure}[t]{0.25\textwidth}
\centering
\includegraphics[height=1.5in]{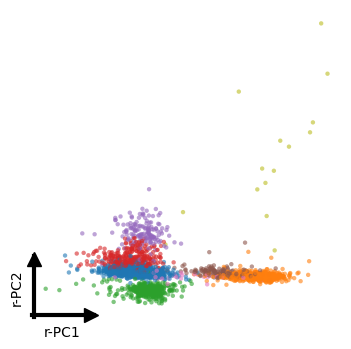}
\caption{PCA}
\end{subfigure}%
\begin{subfigure}[t]{0.25\textwidth}
\centering
\includegraphics[height=1.5in]{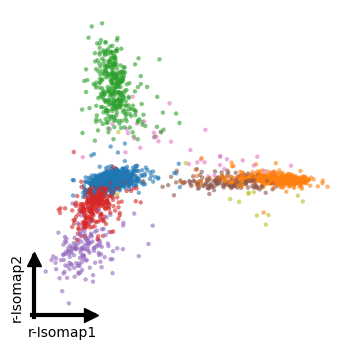}
\caption{Isomap}
\end{subfigure}
\begin{subfigure}[t]{0.25\textwidth}
\centering
\includegraphics[height=1.5in]{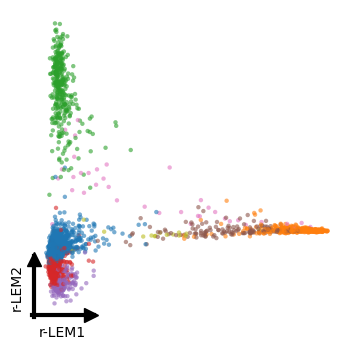}
\caption{Laplacian Eigenmaps}
\end{subfigure}
\caption{Original and robustified embeddings for the PBMC dataset.}
\label{fig:pbmc_embeddings}
\end{figure}

\subsection{The mouse brain from Tabula muris}

The second dataset is another single cell RNA-sequencing dataset,
this time coming from the Tabula muris project \cite{TabulaMuris2020},
sequenced using the Smart-seq2 protocol, which we downsample to the
7,249 non-myeloid cells in the brain tissue sample.  According to the
provided cell type annotation, this sample contains a mixture of
neurons as well as glial cells.  We perform the same experimental
protocol as above on the PBMC data set, and find similar conclusions,
except that in this case Isomap gives a significantly more diffuse
cluster in Procrustes space than PCA or Laplacian eigenmaps, and seems
to give as poor a performance as t-SNE and UMAP.  See
figures~\ref{fig:tm_procrustes_hist} and~\ref{fig:tm_embeddings}.

\ifmam
\else
\begin{figure}
\centering
\begin{subfigure}[t]{0.45\textwidth}
\centering
	\includegraphics[height=0.75in]{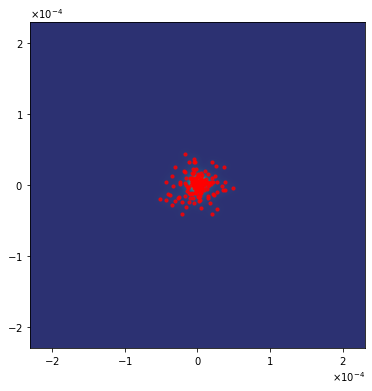}
            \includegraphics[height=0.75in]{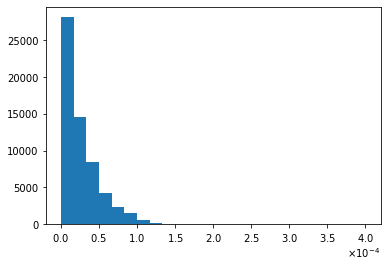}
            \caption{PCA}
        \end{subfigure}%
        \begin{subfigure}[t]{0.45\textwidth}
            \centering
            \includegraphics[height=0.75in]{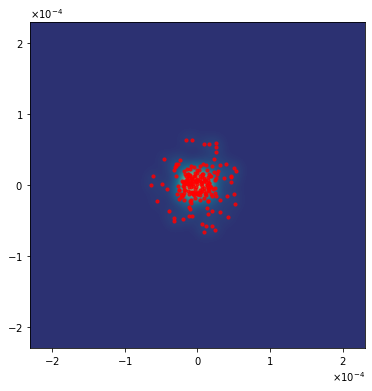}
            \includegraphics[height=0.75in]{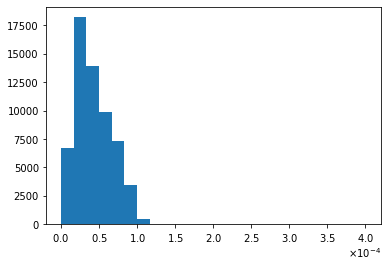}
            \caption{Isomap}
        \end{subfigure}
\\[1ex]
                \begin{subfigure}[t]{0.45\textwidth}
            \centering
            \includegraphics[height=0.75in]{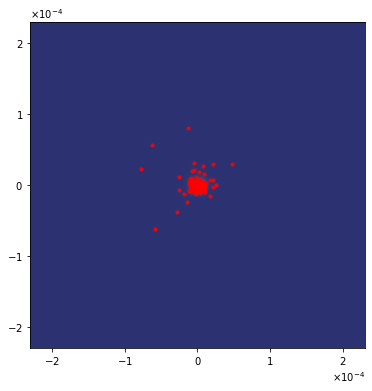}
            \includegraphics[height=0.75in]{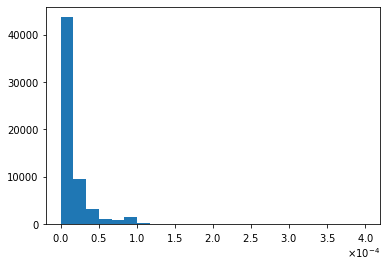}
            \caption{Laplacian eigenmaps.}
        \end{subfigure}%
        \begin{subfigure}[t]{0.45\textwidth}
            \centering
            \includegraphics[height=0.75in]{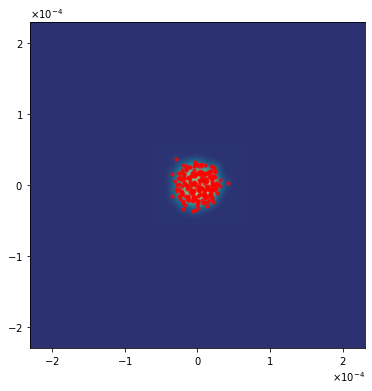}
            \includegraphics[height=0.75in]{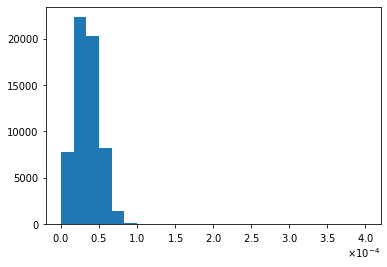}
            \caption{t-SNE.}
        \end{subfigure}%
\\[1ex]
        \begin{subfigure}[t]{0.45\textwidth}
            \centering
            \includegraphics[height=0.75in]{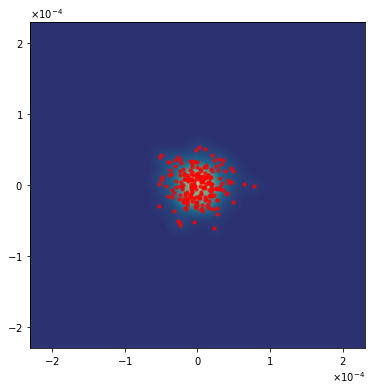}
            \includegraphics[height=0.75in]{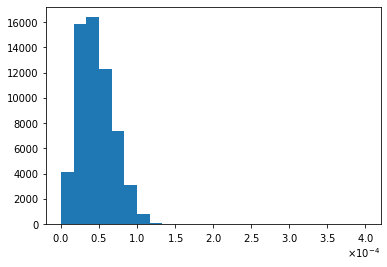}
            \caption{UMAP.}
        \end{subfigure}%

\caption{Procrustes distances between subsample embeddings for the TM dataset.}
\label{fig:tm_procrustes_hist}
\end{figure}
\fi

\begin{figure}
\begin{subfigure}[t]{0.25\textwidth}
\centering
\includegraphics[height=1.5in]{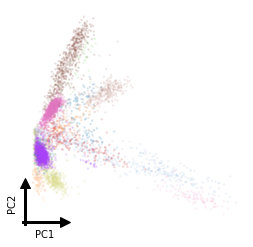}
\end{subfigure}%
\begin{subfigure}[t]{0.25\textwidth}
\centering
\includegraphics[height=1.5in]{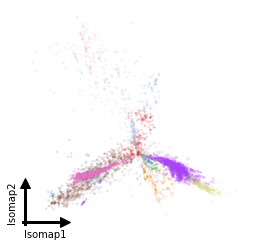}
\end{subfigure}
\begin{subfigure}[t]{0.25\textwidth}
\centering
\includegraphics[height=1.5in]{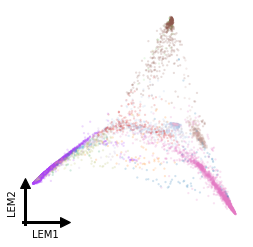}
\end{subfigure}
\\
\begin{subfigure}[t]{0.25\textwidth}
\centering
\includegraphics[height=1.5in]{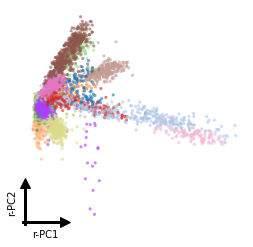}
\caption{PCA}
\end{subfigure}%
\begin{subfigure}[t]{0.25\textwidth}
\centering
\includegraphics[height=1.5in]{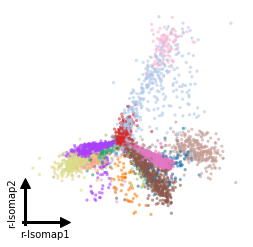}
\caption{Isomap}
\end{subfigure}
\begin{subfigure}[t]{0.25\textwidth}
\centering
\includegraphics[height=1.5in]{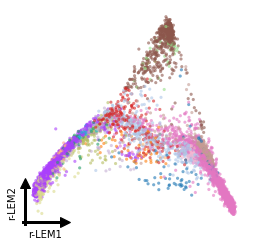}
\caption{Laplacian Eigenmaps}
\end{subfigure}
\caption{Original and robustified embeddings for the TM dataset.}
\label{fig:tm_embeddings}
\end{figure}

\printbibliography

\end{document}